\documentclass[final,12pt]{colt2025} %

\usepackage{comment}
\usepackage{subcaption}
\allowdisplaybreaks

\newcommand{\R}{\mathbb R}

\newcommand{\x}\times
\newcommand{\V}{\mathbb V}
\newcommand{\E}{\mathbb E}
\newcommand{\PP}{\mathbb P}

\newcommand{\cN}{\mathcal N}

\DeclareMathOperator*{\diag}{diag}

\newcommand{\cR}{\mathcal{R}}

\newcommand{\orr}[1]{\overrightarrow{#1}}
\newcommand{\orl}[1]{\overleftarrow{#1}}

\definecolor{ForestGreen}{cmyk}{0.91,0,0.88,0.12}
\colorlet{pierrem}{ForestGreen}

\title[Large Learning Rates in Denoising Score Matching Prevent Memorization]{Taking a Big Step: Large Learning Rates in Denoising Score Matching Prevent Memorization}

\usepackage{times}

\coltauthor{
\Name{Yu-Han Wu}
\Email{yu-han.wu@etu.sorbonne-universite.fr}\\
\addr LPSM, Sorbonne Université
\AND
\Name{Pierre Marion}
\Email{pierre.marion@epfl.ch}\\
\addr Institute of Mathematics, EPFL
\AND
\Name{G{\'e}rard Biau}
\Email{gerard.biau@sorbonne-universite.fr}\\
\addr LPSM, Sorbonne Université, Institut universitaire de France
\AND
\Name{Claire Boyer}
\Email{claire.boyer@universite-paris-saclay.fr}\\
\addr LMO, Université Paris-Saclay, Institut universitaire de France
}
\begin{document}
\maketitle

\begin{abstract}%
Denoising score matching plays a pivotal role in the performance of diffusion-based generative models.
However, the empirical optimal score--the exact solution to the denoising score matching--leads to memorization, where generated samples replicate the training data. Yet, in practice, only a moderate degree of memorization is observed, even without explicit regularization. In this paper, we investigate this phenomenon by uncovering an implicit regularization mechanism driven by large learning rates. Specifically, we show that in the small-noise regime, the empirical optimal score exhibits high irregularity. We then prove that, when trained by stochastic gradient descent with a large enough learning rate, neural networks cannot stably converge to a local minimum with arbitrarily small excess risk. Consequently, the learned score cannot be arbitrarily close to the empirical optimal score, thereby mitigating memorization. To make the analysis tractable, we consider one-dimensional data and two-layer neural networks. Experiments validate the crucial role of the learning rate in preventing memorization, even beyond the one-dimensional setting.
\end{abstract}


\section{Introduction}  \label{sec:intro}
Diffusion models have achieved remarkable success in generative modeling across a wide range of tasks, including computer vision~\citep{amit2021segdiff, baranchuk2021label}, temporal data modeling~\citep{chen2020wavegrad,alcaraz2022diffusion}, multimodal modeling~\citep{ramesh2022hierarchical, rombach2022high}, and natural language processing (\citealp[NLP,][]{austin2021structured, savinov2021step}). Using diffusion models for generative modeling was first proposed by \cite{sohl2015deep}. Subsequently, denoising diffusions reached state-of-the-art performance \citep{song2019generative,ho2020denoising} when trained efficiently with denoising score matching \citep{hyvarinen2005estimation,vincent2011connection}. This training objective consists in learning to denoise artificially perturbed images from the training sample, which is mathematically equivalent to learning the gradient of the log-density, or \textit{score}, of the noisy empirical data distribution. Once score matching has been performed, new observations can be generated by running a backward diffusion process that involves the learned score.
Beyond applications in diffusion, denoising score matching is widely used in various tasks~\citep[see, for example,][]{milanfar2024denoising}, including image restoration~\citep{venkatakrishnan2013plug,teodoro2016image} and nonlinear inverse problems~\citep{wu2019online}. 

However, despite the remarkable effectiveness of denoising score matching, the theoretical properties of this training procedure remain unclear. In particular, if the score matching objective were solved perfectly, meaning that the learned score is equal to the score of the empirical data distribution, then the distribution generated by the diffusion process would exactly coincide with the empirical distribution \citep{li2024good}, a failure mode known as \textit{memorization}. Avoiding this issue is crucial in terms of privacy, intellectual property rights \citep{vyas2023provable, zhang2023copyright}, and ability of models to effectively generate new, unseen data. 

Empirically, only a moderate amount of memorization is observed in practical settings \citep{carlini2023extracting,somepalli2023diffusion,gu2023memorization,kadkhodaie2023generalization}, even without explicitly regularizing the training objective (see Section \ref{sec:related-work} for details on explicit regularization). This suggests the existence of an \textit{implicit regularization} mechanism that prevents exact solving of the denoising score matching problem and thus full memorization of the training data. However, the nature of this regularization is an open problem, as already highlighted in \citet{biroli2024dynamical}.

\paragraph{Contributions.} In this paper, we approach the question through the lens of the (ir)regularity of the empirical optimal score, drawing a connection with the learning rate of stochastic gradient descent (SGD). 
We first show that the empirical optimal score is irregular in the sense that its derivative has a large (weighted) total variation (Section~\ref{sec:regularity-empirical-score}).
Building upon a recent literature on the impact of the learning rate on the regularity of stable minima for SGD \citep{mulayoff2021implicit,qiao2024stableminimaoverfitunivariate}, we prove that the empirical optimal score cannot be stably learned by SGD unless the learning rate becomes vanishingly small (Section~\ref{sec:main}).
Our main result can be informally stated as follows.
\begin{theorem}
\label{thm:intro}(informal)
Consider the denoising score matching objective $\mathcal{R}_n$ over the class of two-layer neural networks for one-dimensional data.  
Then, for a sufficiently small level of noise $\sigma$ and a learning rate $\eta \gtrsim \sigma^2$, the stochastic gradient descent on $\mathcal{R}_n$ cannot stably converge to a local minimum with arbitrarily small excess risk.
\end{theorem}
The take-home message is that for $\sigma$ small enough and $\eta$ large enough, the learned score cannot be arbitrarily close to the empirical optimal one.
To the best of our knowledge, this is the first demonstration of an implicit regularization mechanism in denoising score matching that prevents (full) memorization, thanks to the non-vanishing learning rate used in practice.
Focusing on small values of $\sigma$ is reasonable, as they correspond to the last steps of the backward diffusion, which are known to play a critical role in memorization \citep{raya2023spontaneous,biroli2024dynamical}.

We state our result with SGD because the objective $\mathcal{R}_n$ writes as an expectation, hence requiring stochastic approximation. However, our proof also carries over to GD on the population risk---see Appendix \ref{app:comments-gradient-descent} for further comments. Therefore, 
the non-convergence of SGD towards the global minimizer is not due to the lack of handling the variance of the gradient estimates, but rather to an (implicit) bias due to the large learning rate.

Our results are illustrated through experiments in Section~\ref{sec:experiments}, supporting the connection between the choice of learning rate and memorization, and suggesting that our findings extend beyond the one-dimensional case. %
Some open questions are discussed in Section \ref{sec:conclusion}.

\section{Related work}
\label{sec:related-work}
\paragraph{Implicit bias of large learning rates and minima stability.} 
Large learning rates are an essential implicit regularization mechanism in deep learning \citep[see, e.g.,][]{li2019towards,andriushchenko2023sgd}. In particular, the learning rate provides an upper bound on the maximal eigenvalue of the Hessian of the risk at a twice-differentiable stable minimum \citep{wu2018sgd}. This result can be extended to non-differentiable minima for underparameterized networks \citep{mulayoff2021implicit}. In addition, \citet{qiao2024stableminimaoverfitunivariate} prove a generalization bound for twice-differentiable stable minima in regression tasks, by relating the condition on the Hessian to the functions that can be represented by the neural network. In the present paper, we investigate the impact of large learning rates in the setting of denoising score matching. An important feature of our analysis is that the noise in the training objective of score matching has a regularizing effect, as highlighted by Lemma \ref{lem:twice-differentiability}. This regularizing effect guarantees the necessary assumption that SGD reaches a twice-differentiable local minimum. This contrasts with the standard regression case, where this hypothesis is hardly met (indeed, in general, SGD with ReLU networks tends to align the kinks with the data points; see, e.g., \citealt{boursier2023penalising}, resulting in lower regularity—$C^1$ instead of $C^2$). Importantly, our proof technique should be adaptable to other tasks with noise in the training objective, for example in robust learning or dropout regularization.

\paragraph{Memorization effect of diffusion models.} Diffusion models were found to generate replicas of their training data \citep[see, e.g.,][]{carlini2023extracting,somepalli2023diffusion, somepalli2023understanding}, raising privacy and security concerns. Following these initial observations, a series of papers quantified this memorization phenomenon \citep{gu2023memorization,yoon2023diffusion,kadkhodaie2023generalization}. These articles experimentally demonstrate a transition from memorization to generalization as the sample size increases, showing that with practical sample sizes, the extent of memorization is limited. Furthermore, \citet{kadkhodaie2023generalization} link the generalization ability of diffusion models to their adaptability to the underlying geometric structure of the data. Finally, \citet{gu2023memorization}, \citet{yi2023generalization}, and \citet{li2024good} show that diffusion models with the empirical optimal score exhibit full memorization.  

\paragraph{Regularization of denoising score matching.}
In practice, several methods can be used to mitigate memorization. 
Regularization techniques like weight decay, dropout, or data corruption, can help reduce the model’s dependency on specific data points \citep{daras2024ambient,gu2023memorization,baptista2025memorization}.  All these methods rely on explicitly regularizing the training process. On the contrary, the present paper studies the \textit{implicit} regularization effect of the learning rate in denoising score matching, in order to explain the moderate amount of memorization observed in practice even without explicit regularization. The work by \citet{zeno2024minimum} is more closely aligned with our approach. They derive a closed-form formula for the minimum-norm interpolator of the 1d denoising problem and analyze its generalization properties. We adopt a complementary approach, focusing on SGD stability rather than interpolation and minimum-norm representation. Finally, our analysis supports experimental evidence by \citet{li2024understanding}, who observe that diffusion models capable of generalization tend to learn near-linear scores. Indeed, we show that the learning rate constrains the learned score’s nonlinearity (via the total variation of its derivative), thus preventing full memorization.

\section{Denoising score matching at a glance}
\label{sec:denoising-glance}
In this section, we define the problem of denoising score matching and its connection with diffusion-based generative models. %
\paragraph{Diffusion-based generative models.} 
Let $p_{\textnormal{true}}$ be an unknown non-atomic distribution on $\mathbb{R}$ with finite variance.
Diffusion-based generative models aim to generate new observations following $p_{\textnormal{true}}$, given an i.i.d.~sample of $p_{\textnormal{true}}$. The principle is as follows. For $t \in [0, T]$, the forward diffusion
\begin{equation}
\label{eq:sde-forward}
d\orr{X}_t = -\orr X_tdt+\sqrt2d\orr B_t, \quad \orr X_0 \sim p_{\textnormal{true}},
\end{equation}
can be reversed in time using the backward diffusion
\begin{equation}
\label{eq:sde-backward}
d\orl X_t = (\orl X_t+2\nabla \log p_{T-t}(\orl X_t))dt+\sqrt2d\orl B_t, \quad \orl X_0 \sim  p_T,
\end{equation}
where $p_t$ is the probability density of $\orr X_t$, and $\orr B_t$ (resp. $\orl B_t$) is a Brownian motion. Note that $\orr{X}_t$ has a density since it is convolved with Gaussian noise, and, according to Tweedie's formula \citep{robbins1956empirical}, the log-density of $\orr{X}_t$, i.e. $\log p_t$, is differentiable. Here and in the following, in a slight abuse of notation, the same notation refers to the distribution and its density. The time reversal means that $\orl X_{T-t}$ has the same distribution as $\orr X_t$. Thus, assuming that sampling from $p_T$ is straightforward, the goal is to use the backward equation \eqref{eq:sde-backward} to generate new observations. However, this requires learning the unknown \textit{score function} $\nabla\log p_t$. To do so, a key observation is that
\begin{equation}
\label{eq:sde-sol}
\orr X_t \overset{\mathcal D}{=} \mu(t)\orr X_0+\sigma(t)\xi,\quad \xi\sim\cN(0, 1),
\end{equation}
where $\mu(t) = e^{-t}$ and $\sigma(t)=\sqrt{1-e^{-2t}}$. Therefore, learning the score $\nabla\log p_t$ for every~$t$ is equivalent to learning the score of the convolution $(\mu(t)p_{\textnormal{true}})*\cN(0, \sigma^2(t))$ for every~$t$. An efficient method to do so is denoising score matching, which we introduce next. 

\paragraph{Denoising score matching and empirical optimal score.} 
In the following, we drop the time index~$t$ to cast the problem in the more general context of denoising. Let $X$ be a real-valued random variable of unknown distribution, keeping in mind that in the diffusion model above, $X$ corresponds to $\orr X_0$. Let $\mu,\sigma \in (0, 1)$ be two real numbers, and $Y$ the random variable defined as $Y=\mu X+\sigma \xi$, where $\xi$ is standard Gaussian noise independent of $X$. In particular, in the diffusion context \eqref{eq:sde-sol}, we have, at any given time $t$, $Y=\orr X_t$, $\mu=e^{-t}$, and $\sigma=\sqrt{1-e^{-2t}}$. We let $p_{\mu, \sigma}$ be the density of~$Y$.

The key to connect denoising, i.e, learning the conditional expectation function $\mathbb{E}[X|Y=y]$, and the score function $\nabla\log p_{\mu, \sigma}(y)$, is that, as shown by \cite{robbins1956empirical} and \citet{miyasawa1961empirical}, $\E[X|Y] = \frac{1}{\mu}(Y+\sigma^2 \nabla\log p_{\mu, \sigma}(Y))$, and thus
\begin{equation}    \label{eq:theoretical-score-matching-objective}
\nabla\log p_{\mu, \sigma} %
\in \underset{s\in L_2}{\mathrm{argmin}}\
\mathbb{E}\big[ (s(Y)+\frac{1}{\sigma^2}{(Y-\mu X)})^2 \big].     
\end{equation}
This variational characterization, called \textit{denoising score matching} \citep{vincent2011connection}, is well-posed since $\nabla\log p_{\mu, \sigma}$ is in $L_2$ provided that $p_{\text{true} }$ has finite variance \citep[see, e.g.,][Lemma~6]{Benton2024Nearly}. Since the distributions of $X$ and $Y$ are unknown, this minimization problem is not directly solvable. Thus, given a sample $x_1, \ldots, x_n$  drawn from~$X$, we instead consider the risk 
\begin{equation}
\label{eq:score-matching-objective}
\mathcal{R}_n(s) =\frac1n\sum_{i=1}^n \E_{Y\sim \mathcal{N}( \mu x_i , \sigma^2)}\big[(s(Y)+\frac1{\sigma^2}(Y-\mu x_i ))^2\big],
\end{equation}
where, for clarity, we use lowercase $x_1, \hdots, x_n$ to indicate that the expression is conditional on the sample. We emphasize that this risk is semi-empirical, as it retains an expectation with respect to the noise. This is in line with practice, where fresh noise is introduced at each step, ensuring that stochastic gradient descent indeed minimizes \eqref{eq:score-matching-objective}. Exploiting the convexity of $\mathcal R_n$ with respect to $s$, one can show \citep{gu2023memorization,li2024good} that its minimizer over all measurable $L_2$ functions is
\begin{equation}
\label{eq:optimal-score-function}
    s^\star(y;\mu,\sigma) = \frac{\sum_{i=1}^n(\mu x_i -y)e^{-\frac{(y-\mu x_i )^2}{2\sigma^2}}}{\sigma^2\sum_{i=1}^n e^{-\frac{(y-\mu x_i )^2}{2\sigma^2}}}, \quad y \in \mathbb R.
\end{equation}
It is interesting to note the resemblance between $s^{\star}(y;\mu,\sigma)$ and a Nadaraya-Watson kernel estimator \citep[][Chapter 7]{gyorfi2006distribution}. Throughout, we refer to the function $s^\star$ as the \textit{empirical optimal score}. We emphasize again its dependence on the sample $x_1, \dots, x_n$, which justifies the terminology \textit{empirical}. This is not to be confused with the minimizer of the theoretical risk \eqref{eq:theoretical-score-matching-objective}.%

\paragraph{The memorization problem.}

The next logical step is to substitute the empirical optimal score,~$s^\star$, for the unknown theoretical score, $\nabla \log p_t$, to generate new data from $ p_{\textnormal{true}}$ by running the backward diffusion \eqref{eq:sde-backward}. However, \citet{li2024good} showed that this procedure is undesirable as it leads to full memorization of the training data. More precisely, consider the backward diffusion
\begin{equation*}    
d\orl X_t = (\orl X_t+2 s^\star(\orl X_t; \mu(T-t), \sigma(T-t)))dt+\sqrt2d\orl B_t, \quad \orl X_0 \sim  \mathcal{N}(0, 1),
\end{equation*}
run from time $t=0$ to time $t=T-\delta$. Then, according to \citet{li2024good}, the total variation distance between the distribution of $\orl X_{T-\delta}$ and a smoothed version of the empirical measure of the training sample can be made arbitrarily small as $T \to \infty$ and $\delta \to 0$. In practice, however, only a moderate degree of memorization is observed (see Section \ref{sec:related-work}). The solution to this puzzle is that practitioners do not use the explicit form of $s^\star$, but perform stochastic gradient descent (SGD) to minimize the risk~\eqref{eq:score-matching-objective}. Our goal is to substantiate this observation by showing that the score fitted with SGD with a large learning rate deviates from $s^\star$, thereby avoiding full memorization.

\section{Network class, training algorithm, and stability}
\label{sec:problem-setups}

\paragraph{Model.} In practice, the risk $\mathcal R_n$ is optimized over a  class of parameterized functions. We consider in the present paper two-layer ReLU networks with $m$ hidden neurons, i.e., a class $\mathcal S$ of the form
\[
\mathcal{S} = \Big\{s_\theta:\R\to\R : s_\theta(y) = \frac1m\sum_{\ell=1}^m w_\ell^{(2)}\phi(w_\ell^{(1)}y+ b_\ell), w_\ell^{(1)}\in\{\pm1\}, (w_\ell^{(2)},b_{\ell})\in [-A, A] \times \R %
\Big\},
\]
where $\theta=(w_{1:m}^{(2)}, b_{1:m})\in\R^{2m}$ and $\phi$ is the ReLU activation function.  At training time, we consider a random initialization of the inner weights $w_\ell^{(1)}$ with values $\pm1$, keeping them fixed during training. 
This simplification is introduced for technical reasons, specifically to avoid non-differentiability issues when an inner weight vanishes. Due to the homogeneity of ReLU, this constraint does not affect the expressivity of $\mathcal{S}$. We also constrain the outer weights $w_\ell^{(2)}$ within a ball. 
This is a mild requirement insofar as $A$ can be chosen arbitrarily large (provided it grows polynomially with~$1/\sigma$). In particular, one can ensure by taking $A$ large enough that the empirical optimal score~$s^\star$ is approximated by functions in $\mathcal{S}$. More precisely, by Lemma~\ref{lem:asymptotic-alpha} in Appendix \ref{sec:technical-lemma}, one has $\int_\R|s^\star{}''(y;\mu,\sigma)|dy \leqslant \frac{C_n}{2 \sigma^6}$ for some explicit $C_n \geqslant 1$ depending on $\mu$ and on the training sample. Accordingly, \citet[][Section 9.3.3]{bach2024learning} shows that~$s^\star$ may be approximated by $\mathcal{S}$ as soon as $A\geqslant \frac{C_n}{\sigma^6}$. To fix ideas, a safe choice is therefore $A= \frac{C_n}{\sigma^6}$.
In addition, for simplicity, we denote the risk for $s_\theta$ as $\mathcal{R}_n(\theta)$ instead of~ $\mathcal{R}_n(s_\theta)$.

\paragraph{Training.} Denoising score matching is performed by minimizing the objective $\mathcal{R}_n$ using SGD. 
Since $\mathcal{R}_n(\theta)$ is not directly computable, we use at each iteration $j$ an unbiased estimator, given by
\begin{equation*}
\hat{\mathcal R}_j(\theta) = \sum_{i\in\mathcal{B}_j}\big(s_\theta(Y_i)+\frac{1}{\sigma^2}(Y_i-\mu x_i)\big)^2,
\end{equation*}
where $\mathcal{B}_j$ is a random subset of $\{1, \hdots, n\}$ and $Y_i \sim \mathcal{N}(\mu x_i, \sigma^2)$. 
Then, SGD updates are obtained as $\theta_{j+1}=\theta_j-m\eta \nabla \hat{\mathcal R}_j(\theta_j)$, where $\eta > 0$ denotes the learning rate. Note that the network output in the definition of~$\mathcal{S}$ and the learning rate are rescaled depending on the width $m$. As shown in \citet{chizat2019lazy,yang2021tensor}, these are the correct normalization factors in the feature learning regime \citep{chizat2018global,mei2018mean,rotskoff2018neural}.

\paragraph{Linearly-stable minima.} Our analysis is based on studying the stability of the sequence $(\theta_j)$ around a local minimum of the empirical risk $\cR_n$, which allows us to link the Hessian of the risk with the learning rate. More precisely, following \citet{mulayoff2021implicit} and \citet{qiao2024stableminimaoverfitunivariate}, the second-order Taylor expansion of $\hat\cR_j$ around a twice-differentiable local minimum $\theta^\star$ of $\cR_n$ is
\[
\hat\cR_j(\theta) \approx \hat\cR_j(\theta^\star) + (\theta-\theta^\star)^\top\nabla\hat\cR_j(\theta^\star)+\frac12(\theta-\theta^\star)^\top\nabla^2\hat\cR_j(\theta^\star)(\theta-\theta^\star),
\]
Therefore, for $\theta_j$ close to $\theta^\star$, this motivates considering the linearized SGD updates
\begin{equation}
\label{eq:sgd-linearized}
\theta_{j+1} = \theta_j-m\eta(\nabla\hat\cR_j(\theta^\star)+\nabla^2\hat\cR_j(\theta^\star)(\theta_j-\theta^\star)).
\end{equation}
It is emphasized that this linearization is valid for $\theta^\star$ located in the interior of the constraint set, i.e., such that $w_\ell^{(2)}\in (-A, A)$ for $1 \leqslant \ell \leqslant m$, an assumption that is made throughout the paper. In this context, a local minimum $\theta^\star$ is said to be \textit{linearly stable} if there exists some $\varepsilon > 0$  such that, for any  $\theta_0$ in the $\varepsilon$-ball $\mathcal{B}_\varepsilon(\theta^\star)$, the following condition holds:
\[
\underset{j\to\infty}{\mathrm{limsup}}\ \E\|\theta_j-\theta^\star\|_2\leqslant \varepsilon,
\]
where $\theta_j$ follows the updates \eqref{eq:sgd-linearized}.
The key property \citep[][Lemma 1]{mulayoff2021implicit} we utilize is that if $\theta^\star$ is linearly stable, then
\begin{equation}    \label{eq:bound-lambdamax-eta}
\lambda_{\max}(\nabla^2 \mathcal{R}_n(\theta^\star))\leqslant\frac2{m\eta},    
\end{equation}
where $\lambda_{\max}$ denotes the largest eigenvalue. 
This result connects the learning rate to the regularity of~$\mathcal{R}_n$: the larger $\eta$, the flatter the empirical risk around stable minima for the linearized SGD~\eqref{eq:sgd-linearized}. 
\section{Regularity of the empirical score function}
\label{sec:regularity-empirical-score}
In this section, we show that the empirical optimal score $s^\star$, defined in \eqref{eq:optimal-score-function}, becomes irregular for small $\sigma$. This analysis represents the first important insight into the memorization phenomenon.

Without loss of generality, we assume that the sample is ordered, i.e., $x_1 \leqslant \dots \leqslant x_n$, and let $\Delta = \min_{2 \leqslant i \leqslant n}(x_i - x_{i-1})$ denote the minimum spacing between consecutive observations. Note that $\Delta$ is a random quantity depending on the sample. We further assume $\Delta > 0$, which is a.s.~the case since $p_{\textnormal{true}}$ is non-atomic. As a typical example, if the $x_i$ are sampled uniformly over an interval of length $a$, then  $\Delta$ is of the order of $a / n^2$ \citep[][]{Molchanov1983, nagaraja2015spacings}.
 
To quantify the regularity of $s^\star$, we first note by standard rules that this function is infinitely differentiable. %
This allows us to resort to the following weighted total variation of the derivative of~$s^\star$ (the superscript $(1)$ reminds us that this quantity is the TV of the \textit{first} derivative of~$s^\star$):
\[
\textnormal{TV}_{\pi}^{(1)}(s^\star) = \int_\R |s^\star{}''(y;\mu,\sigma)| \pi(y;\mu, \sigma)dy,
\]
which, in this context, is interpreted as a measure of the regularity of $s^\star$: the larger the total variation $\textnormal{TV}_{\pi}^{(1)}(s^\star)$, the more the derivative of $s^\star$ fluctuates. %
The weight function~$\pi$ is defined as
\begin{equation*}
\pi(y;\mu,\sigma) = \left\{\begin{array}{cl}
 \mathbb{E}_{\xi\sim\mathcal{N}(0, \sigma^2 )}\big[\min\{\pi^+(y-\xi;\mu,\sigma), \pi^-(y-\xi;\mu,\sigma)\}\big],& \textnormal{if }   \mu x_1  \leqslant y\leqslant\mu x_n  ,\\
 0, & \textnormal{otherwise},
\end{array}\right.\,
\end{equation*}
where, denoting by $U\sim\mathcal{U}(\{x_1, \dots, x_n\})$ a uniform draw from the dataset, for $y\in[ \mu x_1  ,  \mu x_n  ]$, 
\begin{align*}
 \pi^-(y;\mu,\sigma) = \PP(\mu U< y)^2\E[y- \mu U \,|\, \mu  U < y],
 \pi^+(y;\mu,\sigma) = \PP(\mu U> y)^2\E[  \mu U-y \,|\, \mu  U >y].%
\end{align*}
Note that a similar, though distinct, weighting scheme was proposed by \citet{mulayoff2021implicit}. As these authors highlight, $\pi$ puts more weight towards the center of the support of the training data.

The next step consists in rewriting the score $s^\star$ from \eqref{eq:optimal-score-function} into a more probabilistic manner. To do so, let, for $1 \leqslant i \leqslant n$,
$\alpha_i(y; \mu, \sigma) = e^{-\frac{(y-
\mu x_i)^2}{2\sigma^2}} / Z$, where $Z$ normalizes the $\alpha_i$'s to sum to $1$,
and denote by $W(y;\mu, \sigma)$ a random variable taking values in $\{x_1, \hdots, x_n\}$ such that the probability of picking $x_i$ is $\alpha_i(y; \mu, \sigma)$. We arrive at the following identities:
\begin{equation}    \label{eq:sstarprimer}
s^\star(y;\mu,\sigma)=\frac{1}{\sigma^2 }\big(-y+  \mu\mathbb{E}[W(y;\mu,\sigma)]\big) 
\quad \textnormal{and} \quad 
s^\star{}'(y;\mu,\sigma) = \frac1{\sigma^2}\big(-1+\frac{\mu^2}{\sigma^2}\V[W(y;\mu,\sigma)]\big).    
\end{equation}  
The proof of the second identity is given in Lemma \ref{lem:derivative-and-gradient} in the Appendix.
The appeal of this probabilistic formalism is that it connects the properties of $s^\star$ to the moments of $W$. The latter are the topic of the next proposition. All proofs are postponed to the Appendix (except for Theorem \ref{thm:empirical-score-irrgularity}).
\begin{proposition}
\label{prop:variance-bound}
Let $m_i=  \frac{\mu(x_i +x_{i+1})}{2}$, $1\leqslant i\leqslant n-1$. Then 
\[
\V[W(m_i;\mu, \sigma)]\geqslant \frac{1}{2n}(x_i -x_{i+1})^2.
\]
If, in addition, $|y-\mu x_i|\leqslant\frac{\mu}4\Delta$ and $\Delta\geqslant2\frac{\sigma}{\mu}$, then
\[
|\E[W(y;\mu, \sigma)]-x_i |\leqslant n\Delta e^{-\frac{\mu^2\Delta^2}{4\sigma^2}} \quad \textnormal{and} \quad \V[W( y;\mu,\sigma )] \leqslant 4n^2\Delta^2e^{-\frac{\mu^2\Delta^2}{4\sigma^2}}.
\]
\end{proposition}
Note that the lower bound for $\V[W(m_i;\mu,\sigma)]$ is independent of $\sigma$, whereas the upper bound of $\V[W(y;\mu, \sigma)]$ decreases to $0$ as $\sigma \to 0$. This remark translates into a non-vacuous lower bound on the variation of $s^\star{}'$ in the following corollary.
\begin{corollary}
\label{cor:regularity-empirical-optimal-score}
If $\Delta\geqslant2\frac{\sigma}{\mu}$, we have, for $y\in\{x_i, x_{i+1}\}$,
\[
|s^\star{}'( y;\mu,\sigma) - s^\star{}'( m_i;\mu,\sigma)| \geqslant \frac{\mu^2}{\sigma^4}\Big(\frac{(x_i-x_{i+1})^2}{2n}-4n^2\Delta^2e^{-\frac{\mu^2\Delta^2}{4\sigma^2}}\Big).
\]
In addition, 
\begin{equation}
\label{eq:empirical-optimal-loss}
\mathcal{R}_n(s^\star) \leqslant \frac{4 \mu^2 (x_n-x_1)^2}{\sigma^4} e^{-\frac{\mu^2 \Delta^2}{32 \sigma^2}}.
\end{equation}
In particular, the upper bound of \eqref{eq:empirical-optimal-loss} converges to 0 as $\sigma\to0$.
\end{corollary}
In the diffusion-based generative models \eqref{eq:sde-forward}--\eqref{eq:sde-sol}, we have $\sigma = \sqrt{1 - e^{-2t}}$. Thus the condition $\Delta \geqslant 2\frac{\sigma}{\mu}$ is satisfied when the diffusion time $t$ is close to $0$, that is, at the last steps of the backward diffusion. As discussed in Section \ref{sec:intro}, this part of the diffusion plays a key role in memorization. The condition $\Delta \geqslant 2\frac{\sigma}{\mu}$ can be interpreted as the fact that the noisy data points should remain far from each other. A similar small-noise setting has been previously explored in related work \citep[][Assumption 1]{zeno2024minimum}.
In this context, the corollary implies that the loss tends to $0$ as~$t \to 0$. 

Equipped with this foundation, we are now ready to establish a lower bound on~$\textnormal{TV}_{\pi}^{(1)}(s^\star)$.
\begin{theorem}
\label{thm:empirical-score-irrgularity}
If $16n^3e^{-\frac{\mu^2\Delta^2}{4\sigma^2 }}\leqslant 1$, then
$\displaystyle
\,\textnormal{TV}_{\pi}^{(1)}(s^\star)\geqslant \frac{\mu^3n\Delta^3}{2^{12}\sigma^4}.
$
\end{theorem}
\begin{proof}
The lower bound on $\textnormal{TV}_{\pi}^{(1)}(s^{\star})$ is obtained by decomposing the integral over half-intervals between successive datapoints. Using the triangular inequality, we have
\begin{align*} 
\textnormal{TV}_{\pi}^{(1)}(s^\star)
&= \int_{\mu x_1}^{\mu x_n} |s^\star{}''(y; \mu, \sigma)| \pi(y;\mu,\sigma)dy \\
&= \sum_{i=1}^{n-1} \int_{\mu x_i}^{m_i} |s^\star{}''(y; \mu, \sigma)| \pi(y;\mu,\sigma)dy + \int_{m_i}^{\mu x_{i+1}} |s^\star{}''(y; \mu, \sigma)| \pi(y;\mu,\sigma)dy \\
&\geqslant \sum_{i=1}^{n-1} \Big(\min_{y\in[  \mu x_i , \mu x_{i+1}]} \pi(y;\mu,\sigma)\Big) \Big(\int_{\mu x_i}^{m_i} |s^\star{}''(y; \mu, \sigma)| dy + \int_{m_i}^{\mu x_{i+1}} |s^\star{}''(y; \mu, \sigma)| dy\Big) \\
&\geqslant \sum_{i=1}^{n-1} \Big(\min_{y\in[  \mu x_i , \mu x_{i+1}]} \pi(y;\mu,\sigma)\Big) \Big(\Big|\int_{\mu x_i}^{m_i} s^\star{}''(y; \mu, \sigma) dy\Big| + \Big|\int_{m_i}^{\mu x_{i+1}} s^\star{}''(y; \mu, \sigma)dy\Big|\Big) .
\end{align*}
Then, by the fundamental theorem of calculus,
\begin{align*}
\textnormal{TV}_{\pi}^{(1)}(s^\star) &\geqslant \sum_{i=1}^{n-1}\Big(\min_{y\in[  \mu x_i , \mu x_{i+1}]} \pi(y;\mu,\sigma)\Big) \Big(\big|s^\star{}'(m_i; \mu, \sigma)-s^\star{}'(\mu x_i; \mu, \sigma)\big| \\
&\qquad +\big|s^\star{}'(m_i; \mu, \sigma)-s^\star{}'(\mu x_{i+1}; \mu, \sigma)\big|\Big) \\
&\geqslant \sum_{i=1}^{n-1}\frac{\mu}{n^2}\Big(\frac12-e^{-\frac{  \mu^2\Delta^2}{2\sigma^2 }}\Big)\min\Big(\frac{i^2(i-1)}{2}, \frac{(n-i)^2(n-i+1)}{2}\Big)\Delta\\
&\quad\times 2 \frac{\mu^2}{\sigma^4}\Big(\frac{(x_i -x_{i+1})^2}{2 n} - 4n^2\Delta^2e^{-\frac{\mu^2\Delta^2}{4\sigma^2 }}\Big),
\end{align*}
where, in the last inequality, we used a lower bound on $\pi$ given in Proposition \ref{prop:weight-function-lower-bound} in Appendix~\ref{apx:auxiliary-props}, combined with Corollary \ref{cor:regularity-empirical-optimal-score} (notice that the condition $16n^3e^{-\frac{\mu^2\Delta^2}{4\sigma^2}}\leqslant1$ implies $\Delta\geqslant 2\frac{\sigma}{\mu}$).  
By considering only the $i$'s that are between $\lceil n/4 \rceil + 1$ and $\lfloor 3n/4 \rfloor - 1$, we see that
\[
\min\Big(\frac{i^2(i-1)}{2}, \frac{(n-i)^2(n-i+1)}{2}\Big) \geqslant \frac{n^3}{128}.
\]
Then, we obtain that
\begin{align*}
\textnormal{TV}_{\pi}^{(1)}(s^\star) &\geqslant \sum_{i=\lceil n/4 \rceil + 1}^{\lfloor 3n/4 \rfloor - 1} \frac{ \mu^3n\Delta}{64\sigma^4}\Big(\frac12-e^{-\frac{\mu^2\Delta^2}{2\sigma^2 }}\Big)\Big(\frac{\Delta^2}{2n} - 4n^2\Delta^2e^{-\frac{\mu^2\Delta^2}{4\sigma^2 }}\Big)\\
&\geqslant \frac{ \mu^3n\Delta^3}{256\sigma^4}\Big(\frac12-e^{-\frac{\mu^2\Delta^2}{2\sigma^2 }}\Big)\Big(\frac{1}2 - 4n^3 e^{-\frac{\mu^2\Delta^2}{4\sigma^2 }}\Big)\\
&\geqslant \frac{\mu^3 n\Delta^3}{256\sigma^4}\big(\frac12-\frac14\big)\big(\frac12-\frac14 \big),
\end{align*}
where the second inequality utilizes that there are at least $n/4$ points between $\lceil n/4 \rceil + 1$ and $\lfloor 3n/4 \rfloor - 1$ (for $n \geqslant 10$), and the last inequality unfolds from the assumption $16n^3e^{-\frac{\mu^2\Delta^2}{4\sigma^2 }}\leqslant 1$. This concludes the proof.
\end{proof}
The condition 
$16n^3e^{-\frac{\mu^2\Delta^2}{4\sigma^2}}\leqslant1$ 
is equivalent to $\Delta^2\geqslant 4 (\sigma^2/\mu^2) \ln(16 n^3)$. In other words, up to a log factor, the minimum spacing of sample points $\Delta$ is larger than the normalized standard deviation $\sigma/\mu$. Consequently, for a fixed $\mu$, in the small-noise regime $\sigma \to 0$, this condition is satisfied, and Theorem \ref{thm:empirical-score-irrgularity} shows that $s^\star$ becomes more and more irregular. This result can be recast in the diffusion framework \eqref{eq:sde-forward}--\eqref{eq:sde-sol}. Indeed, in this situation where $\mu=e^{-t}$ and $\sigma=\sqrt{1-e^{-2t}}$, the condition of the theorem is satisfied when $t$ is close to 0. For small $t$, we have $\frac{\mu^3}{\sigma^4}= \frac{e^{-3t}}{(1-e^{-2t})^2} \geqslant \frac{1}{8t^2}$, and thus
\[
\textnormal{TV}_{\pi}^{(1)}(s^\star)\geqslant \frac{Cn\Delta^3}{8t^2}.
\]
The conclusion of this section is that $s^\star$ is highly irregular (at least in the sense of the $\textnormal{TV}_{\pi}^{(1)}$ measure) when $\sigma$ is small. This suggests that gradient descent could struggle to learn it, as it is known to exhibit an inductive bias toward learning regular functions \citep[see, e.g.,][Section~12.1]{bach2024learning}. This provides a strong initial argument against the possibility of full memorization. We formalize this intuition in the next section within the context of two-layer neural networks.

\section{Implicit regularization and memorization}
\label{sec:main}
We show that the mechanism of SGD protects the two-layer neural networks defined in Section \ref{sec:problem-setups} from memorization. This goal is achieved in Theorem~\ref{thm:main-without-d2} below by proving that SGD cannot converge to a local minimum $\theta^\star$ with low risk $\mathcal{R}_n(\theta^\star)$, unless the learning rate is small. To establish this result, we study the regularity of $s_{\theta^\star}$ as measured by $\textnormal{TV}_{\pi}^{(1)}(s_{\theta^\star})$ within the linear stability framework outlined in Section \ref{sec:problem-setups}.
This approach requires the risk $\cR_n$ to be twice differentiable at~$\theta^\star$. We start by showing that, in fact, it is twice differentiable everywhere.
\begin{lemma}
\label{lem:twice-differentiability}
For all $\theta = (w_{1:m}^{(2)}, b_{1:m})\in\R^{2m}$, the risk $\cR_n(\theta)$ is twice differentiable with respect to~$\theta$.
\end{lemma}
The following proposition gives an upper bound on $\textnormal{TV}_{\pi}^{(1)}(s_{\theta^\star})$ when $\theta^\star$ is a linearly stable minimum, expressed in terms of the loss and the inverse of the learning rate. Note that, since $s_\theta''$ is a sum of Diracs,  $\textnormal{TV}_{\pi}^{(1)}(s_{\theta})$ is computed in the sense of the theory of distributions (this operation is possible since $\pi(y; \mu,\sigma)$ is a smooth function). 
\begin{proposition}
\label{prop:second-derivative-upper-bound}
Let $\theta = (w_{1:m}^{(2)}, b_{1:m})\in\R^{2m}$. Then
\[
\textnormal{TV}_{\pi}^{(1)}(s_\theta) \leqslant \frac{\lambda_{\max}(\nabla_\theta^2\cR_n(\theta))m}{4} + \frac{\sqrt{\cR_n(\theta)}}{2} + \frac{A}{ \sqrt{2 \pi} \sigma} \max \Big(\sqrt{2 n \mathcal{R}_n(\theta)}, \big( \sqrt{2\pi e} A \sigma n \mathcal{R}_n(\theta)\big)^{\frac{1}{3}} \Big) .
\]
In particular, if $\theta^\star$ is a linearly stable local minimum of $\mathcal{R}_n$, one has
\[
\textnormal{TV}_{\pi}^{(1)}(s_{\theta^\star}) \leqslant \frac{1}{2\eta} + \frac{\sqrt{\cR_n(\theta^\star)}}{2} + \frac{A}{ \sqrt{2 \pi} \sigma} \max \Big(\sqrt{2 n \mathcal{R}_n(\theta^\star)}, \big( \sqrt{2\pi e} A \sigma n \mathcal{R}_n(\theta^\star)\big)^{\frac{1}{3}} \Big) .
\]
\end{proposition}
The proof of the first statement consists in carefully assessing the magnitude of the terms in the Hessian of the risk. A key step is to lower bound the largest eigenvalue of the neural tangent kernel term by $\textnormal{TV}_{\pi}^{(1)}(s_{\theta})$. The second identity of the proposition then directly follows from \eqref{eq:bound-lambdamax-eta}.

Symmetrically, we next provide a lower bound on $\textnormal{TV}_{\pi}^{(1)}(s_{\theta})$ for low-enough values of the risk. 
\begin{proposition}     \label{prop:second-derivative-lower-bound}
If $\Delta\geqslant8\frac{\sigma}{\mu}$, then for any $\theta$ such that $\cR_n(\theta) \leqslant \frac{1}{16n\sigma^2}$, one has
$ 
\textnormal{TV}_{\pi}^{(1)}(s_{\theta}) \geqslant \frac{\mu n^2\Delta}{2^{11} \sigma^2} .
$
\end{proposition}
This result is connected with Theorem \ref{thm:empirical-score-irrgularity}. Indeed, the theorem gives a lower bound on $\textnormal{TV}_{\pi}^{(1)}(s_\theta)$ in the case where $s_\theta = s^\star$. Then, Proposition \ref{prop:second-derivative-lower-bound} relaxes this bound to all neural networks with small enough risk.
Combining the above upper and lower bounds on the $\textnormal{TV}_{\pi}^{(1)}$ metric, we see that a low-risk \textit{and} linearly stable minimum of $\mathcal{R}_n$ imposes a lower bound on $1/\eta$ of the order of $1/\sigma^2$. This observation, combined with Corollary~\ref{cor:regularity-empirical-optimal-score} and elementary  computations, leads to our main result.
\begin{theorem}
\label{thm:main-without-d2}
Let $\theta^\star\in\R^{2m}$ be a linearly stable local minimum of $\cR_n$. %
Then there exists $\sigma_0 > 0$, depending on $\mu$ and the training sample, such that if $\sigma \leqslant \sigma_0$ and $\eta > \frac{2^{12} \sigma^2}{\mu n^2 \Delta}$, one has 
\[
\cR_n(\theta^\star)  - \cR_n(s^\star) > \frac{\pi n^5 \mu^{3} \Delta^{3}}{2^{36} e^{1/2} A^4 \sigma^4} .
\]
\end{theorem}
The main message is that for a fixed (small enough) $\sigma$, if $\eta$ is sufficiently large, then the excess risk $\cR_n(\theta^\star)  - \cR_n(s^\star)$ cannot be made arbitrarily small. This is equivalent to stating that $s_{\theta^\star}$ cannot be arbitrarily close to $s^\star$, as can be seen by reformulating the theorem's conclusion as
\[
\frac{\pi n^5 \mu^{3} \Delta^{3}}{2^{36} e^{1/2} A^4 \sigma^4} < \cR_n(\theta^\star) - \cR_n(s^\star) = \frac1n\sum_{i=1}^n\E_{Y\sim \mathcal{N}(\mu x_i , \sigma^2)}\big[(s_{\theta^\star}(Y) - s^\star(Y;\mu,\sigma))^2\big].
\]
We refer to, e.g., \citet{coste2023diffusions} for a proof of this identity. The right-hand side can be interpreted as a weighted $L_2$ distance between $s_{\theta^\star}$ and $s^\star$, assigning greater weight around the noisy observations. 
In the small-noise regime, i.e., when $\sigma \to 0$, the two conditions of Theorem \ref{thm:main-without-d2} are automatically satisfied. This is true in particular for diffusions \eqref{eq:sde-forward}--\eqref{eq:sde-sol} as $t \to 0$. Therefore, in this context, Theorem \ref{thm:main-without-d2} suggests that setting a large learning rate prevents memorization of the training sample.

\section{Experiments}
\label{sec:experiments}
In this section, we experimentally assess the closeness of the learned model $s_{\theta^\star}$ to the empirical optimal score $s^\star$, as well as the memorization effect, depending on the learning rate and the dimension of the data. Following our theoretical framework, we fix the model to be a $2$-layer ReLU network of width $1000$. In all experiments, the number of epochs scales inversely with the learning rate, to ensure comparable convergence across models. Experimental details are given in Appendix~\ref{apx:experiments}. %

\paragraph{Connection between the learning rate and the proximity of $s_{\theta^\star}$ to $s^\star$.}
For the first experiment, the training data $x_1, \dots, x_{20}$ are 20 i.i.d.~samples of the one-dimensional standard Gaussian. We perform SGD on the score matching risk \eqref{eq:score-matching-objective}, for fixed values of $\mu$ and $\sigma$ (taken so $\mu^2 + \sigma^2 = 1$).
Figure~\ref{fig:exp-plot-function} shows the graphs of the learned models $s_{\theta^\star}$ trained with different learning rates, together with the empirical optimal score~$s^\star$. As expected from our theory (Theorem~\ref{thm:main-without-d2}), we observe that a larger learning rate~$\eta$ or smaller noise variance~$\sigma$ prevents $s_{\theta^\star}$ from converging to $s^\star$. This leads to a larger excess risk, which is confirmed in Figure~\ref{fig:excess-loss-1d} (left). We present in Figure~\ref{fig:excess-loss-1d} (right) the result of the analogous experiment in dimension $10$, highlighting the same pattern. This provides evidence that the model should not (fully) memorize the data, as we further investigate next. We also report the largest eigenvalue of the loss Hessian at the end of training for different learning rates (see Appendix \ref{apx:experiments} Figure \ref{fig:sharpness} (left and middle)), which confirms that taking a larger learning rate leads to convergence to a flatter region of parameter space, which is key behind our analysis.

\begin{figure}[ht]
    \hfill
    \includegraphics[width=0.45\linewidth]{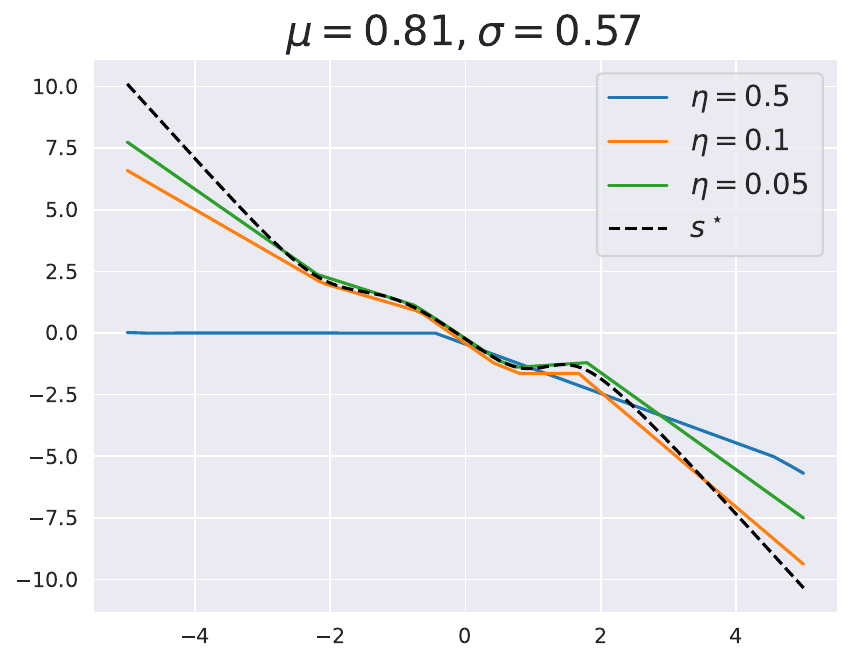}
    \hfill
    \includegraphics[width=0.45\linewidth]{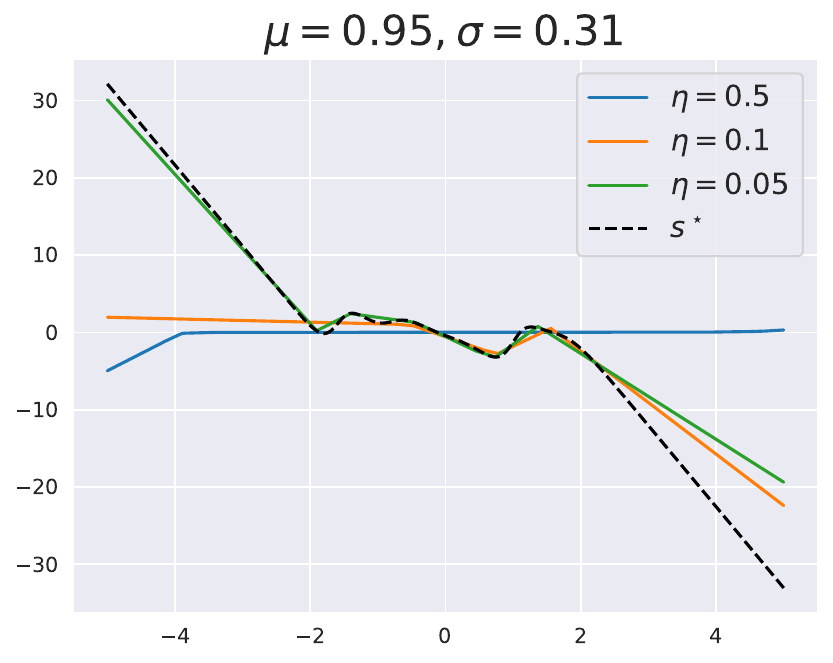}
    \hfill
    \hfill
    \caption{Graphs of the learned model $s_{\theta^\star}$ with different learning rates and of the empirical optimal score~$s^\star$, for two pairs of $(\mu, \sigma)$. As the learning rate decreases, $s_{\theta^\star}$ approaches $s^\star$. When $\sigma$ is smaller (right plot), $s^\star$ is more irregular, and a smaller learning rate is needed for $s_{\theta^\star}$ to approach~$s^\star$.}
    \label{fig:exp-plot-function}
\end{figure}

\begin{figure}[ht]
    \hfill
    \includegraphics[width=0.45\linewidth]{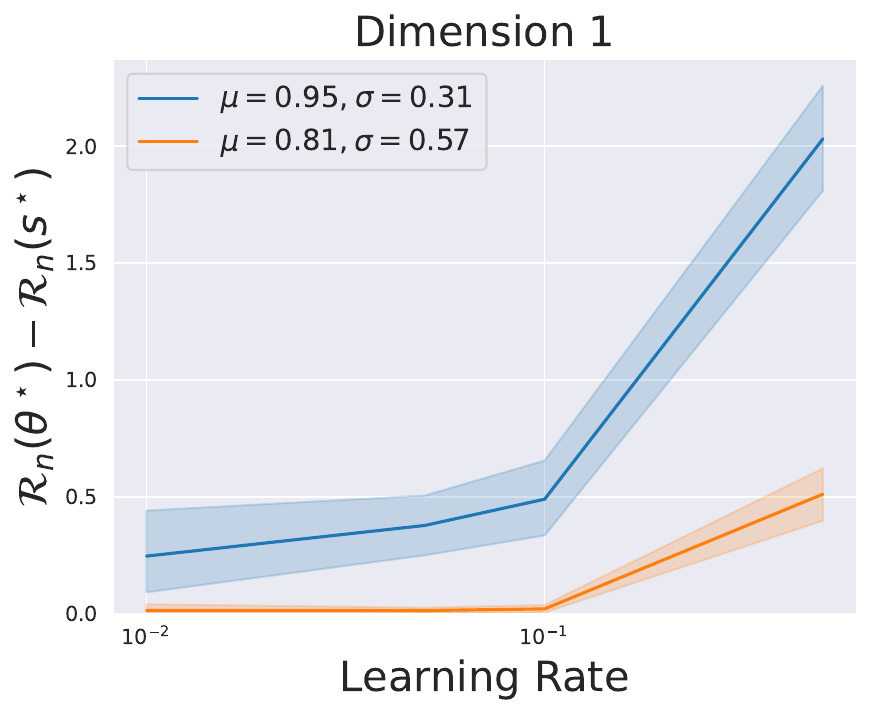}
    \hfill
    \includegraphics[width=0.45\linewidth]{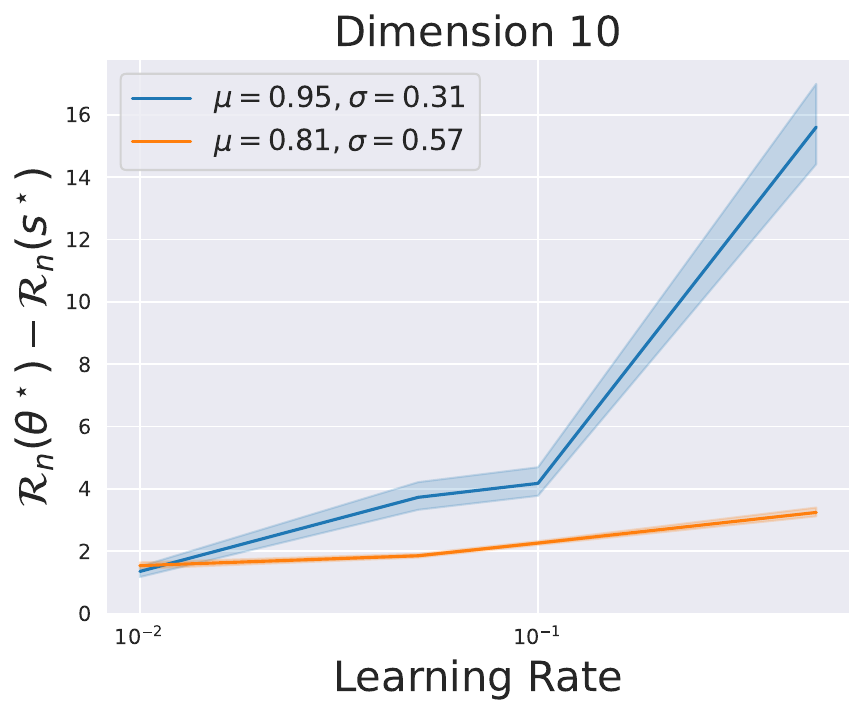}
    \hfill
    \hfill
    \caption{Excess risk of the learned model $s_{\theta^\star}$ trained with different learning rates, for two pairs of $(\mu, \sigma)$ and two dimensions of the data ($d=1$, left, and $d=10$, right). The $x$-axis is in logarithmic scale while the $y$-axis is in standard scale. Confidence intervals are computed with 30 simulations.
    }
    \label{fig:excess-loss-1d}
\end{figure}

\begin{figure}[ht]
    \centering
    \hfill
    \includegraphics[width=0.32\linewidth]{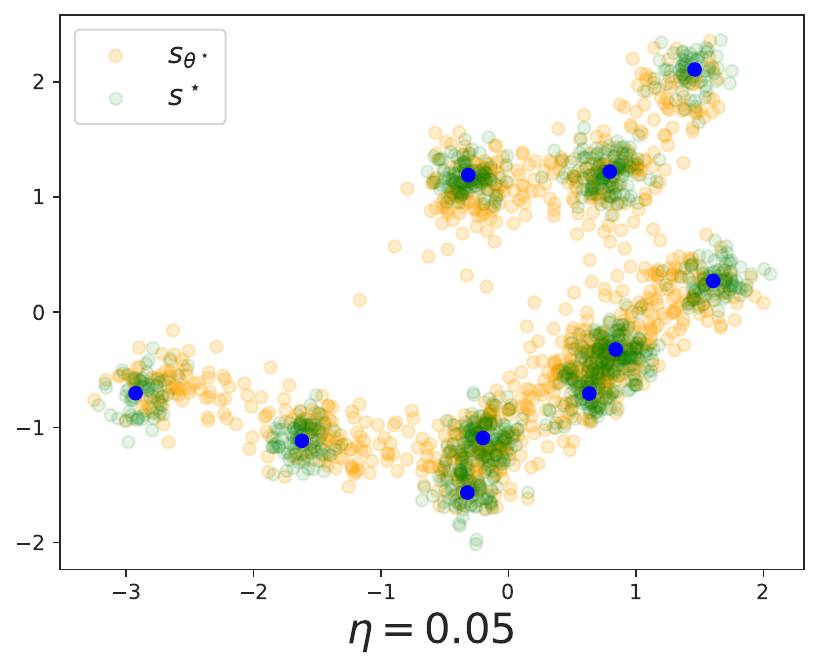}
    \hfill
    \includegraphics[width=0.32\linewidth]{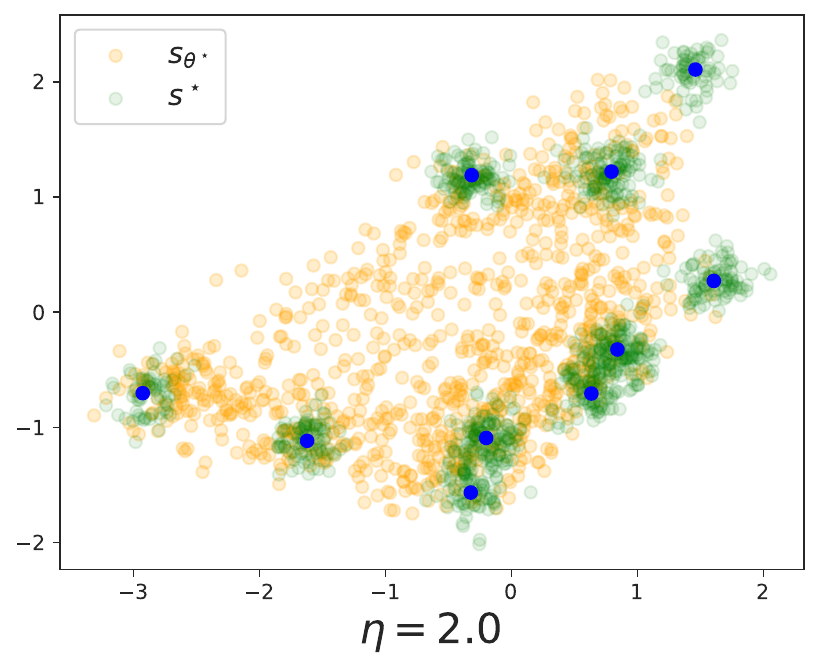}
    \hfill
    \includegraphics[width=0.32\linewidth]{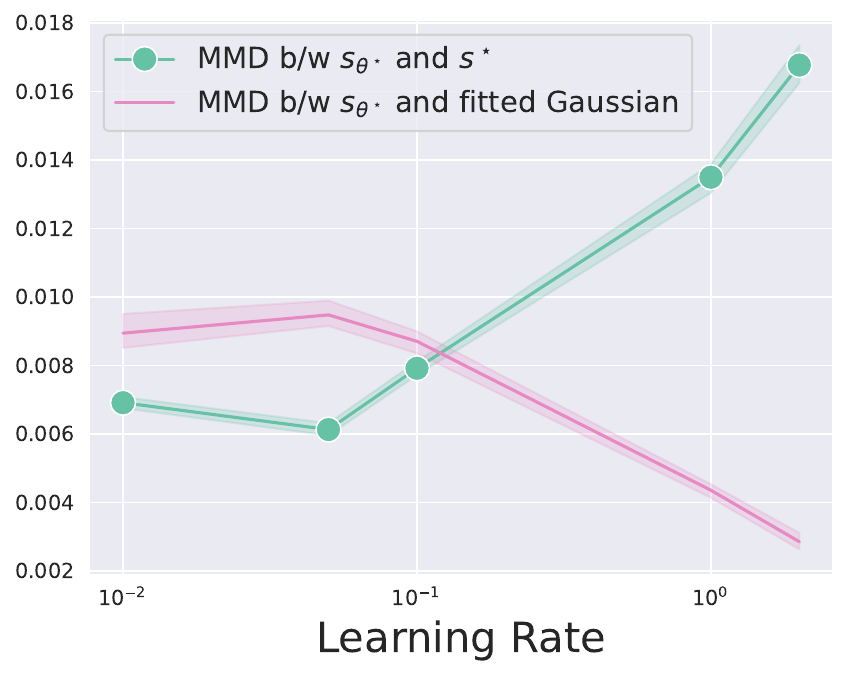}
    \hfill
    \caption{(left) Sample generated by $s^\star$ and $s_{\theta^\star}$ fitted with learning rate $0.05$. The training data are the blue points. (middle) Same with $s_{\theta^\star}$ fitted with learning rate $2$. (right) The green marked curve corresponds to the MMD between observations generated by $s^\star$ and observations generated by $s_{\theta^\star}$ (for different learning rates). The pink curve is the MMD between observations following the Gaussian distribution fitted on the training data and observations generated by $s_{\theta^\star}$.
    }
    \label{fig:exp-2d-memorization}
\end{figure}

\paragraph{Learning rate and memorization effect for denoising diffusions.}
In dimension $d=2$, we sample $10$ isotropic Gaussian observations, and aim to generate new ones using a diffusion model. We perform denoising score matching to learn the score $(t, x)\in\R^{d+1} \mapsto \nabla \log p_t(x)\in\R^d$ with a two-layer neural network $s_\theta(t, x)$ fitted on noisy observations (for various noise variances $\sigma(t)$). Running the diffusion with the learned score, we observe in Figure~\ref{fig:exp-2d-memorization} (left and middle) that a small learning rate leads to generating observations close to the training data, indicating memorization. As expected, simulating the diffusion with~$s^\star$ also induces memorization. In contrast, a larger learning rate leads to observations closer to the target distribution. This is further verified by measuring the maximum mean discrepancy (MMD) between generated and training data (Figure~\ref{fig:exp-2d-memorization}, right). This figure also suggests that a larger learning rate not only avoids memorization but also learns a Gaussian distribution fitted on the training sample. This is not too surprising since larger learning rates constraint the total variation of the derivative (Proposition~\ref{prop:second-derivative-upper-bound}); in the limit where $\textnormal{TV}_{\pi}^{(1)}(s_{\theta^\star}) \to~0$, the model can only implement a linear function, and the optimal linear score generates such a Gaussian distribution. This is also in line with findings of \citet{li2024understanding}---see Section \ref{sec:related-work}. In addition, we report the largest eigenvalue of the loss Hessian at thet end of training for different learning rates in Appendix \ref{apx:experiments} Figure \ref{fig:sharpness} (right).

\paragraph{Dimension and memorization effect for denoising diffusions.}
In this final experiment, we examine the effect of the dimension on memorization, while keeping the learning rate fixed and following the same experimental procedure as previously. Since our results depend on the minimum spacing between data points, which scales with dimensionality, memorization is expected to be less prominent in the high-dimensional regime---a common scenario in image generation. This is confirmed by Figure~\ref{fig:exp-dimension}~(left and middle), which shows that, as the dimension increases, the fitted neural network $s_{\theta^\star}$ generates observations that are less similar to the training observations. Figure \ref{fig:exp-dimension} (right) confirms this finding by measuring the MMD between observations generated by $s_{\theta^\star}$ and observations generated by $s^\star$. This indicates that avoiding memorization is easier in a high-dimensional setting. Further, in high dimension, the network seems to be learning a Gaussian distribution.
\begin{figure}
    \centering
    \hfill
    \includegraphics[width=0.32\linewidth]{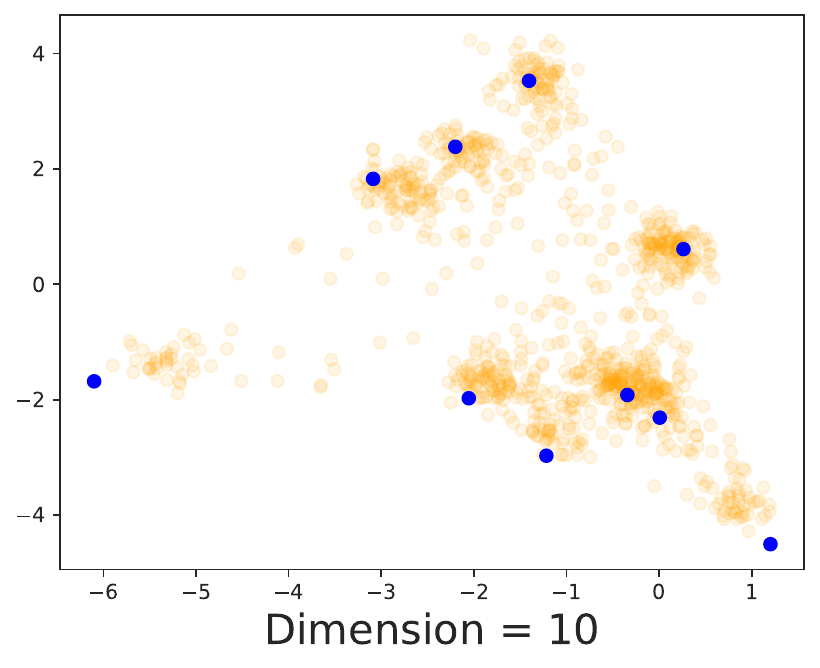}
    \hfill
    \includegraphics[width=0.32\linewidth]{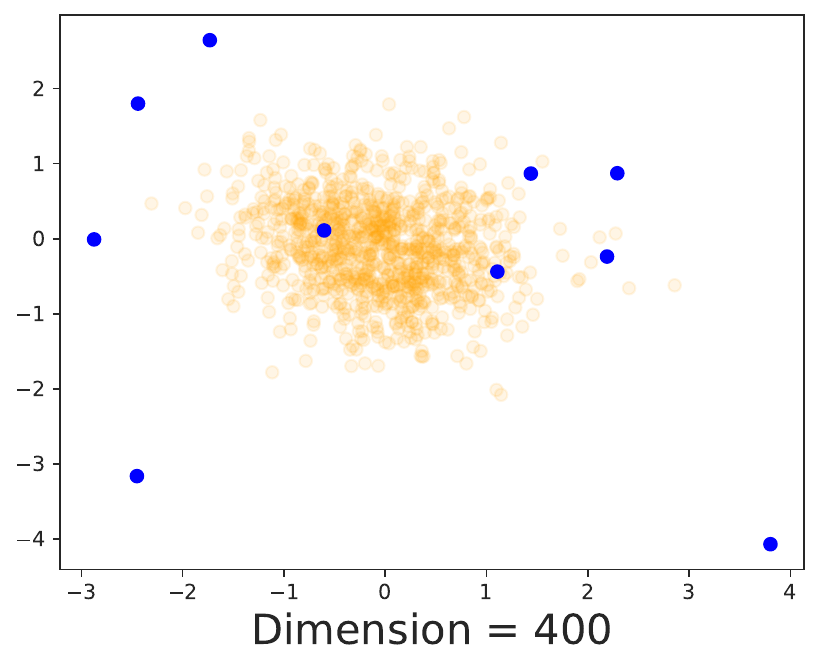}
    \hfill
    \includegraphics[width=0.32\linewidth]{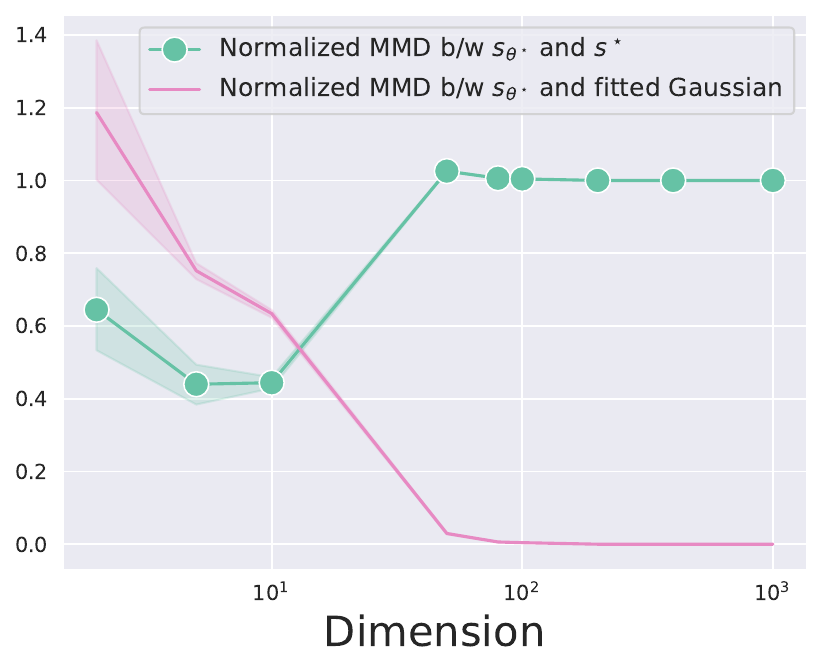}
    \hfill
    \caption{%
    (left) Sample generated by $s_{\theta^\star}$ in dimension $10$, projected on the first two axes. The training data are the blue points. (middle) Same in dimension~$400$. (right) The green marked curve corresponds to the MMD between observations generated by $s^\star$ and observations generated by $s_{\theta^\star}$, depending on the dimension. The pink curve is the MMD between observations following the Gaussian distribution fitted on the training data and observations generated by $s_{\theta^\star}$. Both distances are normalized by the MMD between observations generated by $s^\star$ and by the Gaussian distribution.}
    \label{fig:exp-dimension}
\end{figure}
\section{Conclusion}
\label{sec:conclusion}
In this paper, we provide a theoretical explanation for why neural networks do not fully memorize the training data. Specifically, our main result, Theorem \ref{thm:main-without-d2}, shows that using a large learning rate acts as a regularization mechanism. This mechanism prevents the network from converging to the empirical optimal score, which becomes unstable under stochastic gradient descent. 

Our approach can be extended in several directions. While our theoretical results focus on the one-dimensional case, it is our belief that a good understanding of the one-dimensional problem provides a solid basis for the more complex study of higher dimensional cases, for example by revealing the crucial role played by the data minimal spacing. On top of this, our experiments strongly suggest that similar results hold in a multivariate setting. Future work could explore this extension, for example by leveraging a multivariate version of the minimum stability framework \citep{nacson2023implicit}. Next, our results highlight the critical relationship between the learning rate $\eta$ and the noise variance $\sigma$. However, the effects of the sample size $n$ and the dimension $d$ remain unclear. An interesting question is to analyze the role of these hyperparameters, connecting with the literature discussing the impact of~$n$ (see Section \ref{sec:related-work}).
Finally, even if our results indicate that $\eta$ should not be too small to prevent memorization, they do not guarantee that a larger learning rate improves the quality of the generated data. Exploring trade-offs between memorization, generation quality, and training speed is a key avenue for future research.

\paragraph{Acknowledgments.} Authors warmly thank Florentin Guth for insightful discussions that motivated them towards this line of research, as well as Peter Bartlett, Raphaël Berthier, and Serena Booth for helpful comments. P.M.~is supported by a Google PhD Fellowship.

\bibliography{ref}

\newpage

\appendix

\begin{center}
    \LARGE \textbf{Appendix}
\end{center}

\paragraph{Organization of the Appendix.} Appendix \ref{apx:auxiliary-props} presents two additional propositions of interest preliminary to the proofs of the results of the main text, which are given in Appendix \ref{apx:proofs}. Appendix \ref{sec:technical-lemma} is dedicated to technical lemmas.
Finally, Appendix \ref{apx:experiments} details our experimental setting.

\section{Auxiliary propositions}
\label{apx:auxiliary-props}

The next proposition provides bounds on the weight function $\pi$ defined in Section \ref{sec:regularity-empirical-score}.
\begin{proposition}
\label{prop:weight-function-lower-bound}
Let $1\leqslant i\leqslant n-1$ and $y\in[\mu x_i ,  \mu x_{i+1}]$. Then 
\begin{align*}
 \pi(y;\mu,\sigma)
&\geqslant \frac{\mu}{n^2}\Big(\frac12-e^{-\frac{ \mu^2\Delta^2}{2\sigma^2 }}\Big)\min\Big(\frac{i^2(i-1)}{2}, \frac{(n-i)^2(n-i+1)}{2}\Big)\Delta.
\end{align*}
On the other hand, for any $y\in\R$, 
\[
 \pi(y;\mu,\sigma) \leqslant \mu(x_n-x_1).
\]
\end{proposition}
\begin{proof}
Recall that
\[
 \pi^-(y;\mu,\sigma)=\PP(\mu U<y)^2\E(y-\mu U|  \mu U<y).
\]
\paragraph{Lower bound.} Let $x = \frac{y}{\mu}$. Then, clearly, $y\in[\mu x_i, \mu x_{i+1}]$ is equivalent to $x\in[x_i, x_{i+1}]$. Hence, 
\[
\PP(\mu U < y)^2 = \PP(U<x)^2=\frac{i^2}{n^2}
\]
and
\[
\E[y - \mu U| \mu U < y]=\mu \E[x-U|U<x]=\mu\Big(x-\frac1i\sum_{i'=1}^i x_{i'}\Big).
\]
Therefore, 
\[
 \pi^-(\mu x; \mu, \sigma)=\frac{\mu i^2}{n^2}\Big(x-\frac1{i}\sum_{i'=1}^i x_{i'
}\Big)\geqslant\frac{\mu i}{n^2}\sum_{i'<i}(x_i -x_{i'}),
\]
and, similarly,
\[
 \pi^+(\mu x; \sigma, \mu) = \frac{\mu (n-i)^2}{n^2}\Big(\frac1{n-i}\big(\sum_{i'=i+1}^n x_{i'}\big)-x\Big)\geqslant \frac{\mu(n-i)}{n^2}\sum_{i'>i+1}(x_{i'}-x_{i+1}).
\]
Recall that 
\[
 \pi(\mu x; \mu, \sigma) = \E_{\xi\sim\mathcal{N}(0, \sigma^2)}\min(\pi^-(\mu x-\xi; \mu, \sigma), \pi^+(t,\mu x-\xi; \mu, \sigma)).
\]
Since $\xi$ in the expectation is Gaussian with mean 0, it is symmetric, and we may rewrite $\pi$ as
\[
 \pi(\mu x; \mu,\sigma) = \E_{\xi\sim\mathcal{N}(0, \sigma^2 )}\min(\pi^-(\mu x+\xi; \mu, \sigma), \pi^+(\mu x+\xi;\mu, \sigma)).
\]
Using the previous bounds on $\pi^+$ and $\pi^-$, we obtain
\begin{align*}
 \pi(\mu x;\mu, \sigma)&\geqslant \frac1{\sqrt{2\pi\sigma^2}}\int_{\mu(x_i -x)}^{  \mu(x_{i+1}-x)} \min(\pi^-(\mu x+z; \mu, \sigma), \pi^+(\mu x+z; \mu,\sigma))e^{-\frac{z^2}{2\sigma^2}}dz\\
&\geqslant \frac{\mu}{n^2}\min\Big(i\sum_{i'<i}(x_i -x_{i'}), (n-i)\sum_{i'>i+1}(x_{i'}-x_{i+1})\Big)\\
& \quad \times \PP_{z\sim\mathcal{N}(0, 1)}(\mu(x_i -x)<\sigma z< \mu( x_{i+1}-x)).
\end{align*}
To bound the last term, we use the fact that $0\in(\mu(x_i -x),\mu(x_{i+1}-x))$ and thus
\begin{align*}
\PP(\mu(x_i -x)&<\sigma z<\mu(x_{i+1}-x)) \\
&= \PP_{z\sim\mathcal{N}(0, 1)}(\mu(x_i -x)<\sigma z<0) +\PP_{z\sim\mathcal{N}(0, 1)}(0<\sigma z<  \mu(x_{i+1}-x))\\
&=\PP_{z\sim\mathcal{N}(0, 1)}(0<\sigma z<\mu(x-  x_i) ) +\PP_{z\sim\mathcal{N}(0, 1)}(0<\sigma z<  \mu(x_{i+1}-x))\\
&\geqslant \PP_{z\sim\mathcal{N}(0, 1)}(  \mu(x_{i+1}-x)<\sigma z< \mu(x_{i+1}-x)+\mu(x -  x_i))\\
&\quad+\PP_{z\sim\mathcal{N}(0, 1)}(0<\sigma z<  \mu(x_{i+1}-x)),\\
&=\PP_{z\sim\mathcal{N}(0, 1)}\Big(\sigma z\in (0, \mu(x_{i+1}-x))\cup \big(\mu(x_{i+1}-x), \mu(x_{i+1}- x_i)\big)\Big) \\
&=\PP_{z\sim\mathcal{N}(0, 1)}\big(\sigma z\in (0, \mu(x_{i+1}- x_i)\big)\big).
\end{align*}
Therefore, we have
\begin{align*}
\pi(\mu x;\mu, \sigma)&\geqslant\frac{\mu}{n^2}\min\Big(i\sum_{i'<i}(x_i -x_{i'}), (n-i)\sum_{i'>i+1}(x_{i'}-x_{i+1})\Big)\\
& \quad \times \PP_{z\sim\mathcal{N}(0, 1)}(0<\sigma z<\mu (x_{i+1}-x_i ))\\
&\geqslant \frac{\mu}{n^2}\min\Big(i\sum_{i'<i}(x_i -x_{i'}), (n-i)\sum_{i'>i+1}(x_{i'}-x_{i+1})\Big)\Big(\frac12-e^{-\frac{\mu^2(x_i -x_{i+1})^2}{2\sigma^2 }}\Big).
\end{align*}
In the last inequality, we used a tail bound of the Gaussian distribution for the last inequality, namely, $\PP_{z\sim\cN(0, \sigma^2)}(z\geqslant t)\leqslant e^{-\frac{t^2}{2\sigma^2}}$ \citep[see, for instance,][]{gordon1941values}. To derive the lower bound of the Proposition, note that
\begin{align*}
\sum_{i'<i}(x_i -x_{i'})& \geqslant \sum_{i'=1}^{i-1} (i-i')\Delta=\frac{i(i-1)}2\Delta,\\
\sum_{i'>i}(x_{i'}-x_i) &\geqslant \sum_{i'=i+1}^{n} (i'-i)\Delta=\frac{(n-i)(n-i+1)}2\Delta.
\end{align*}
\paragraph{Upper bound.} Again, let $x = \frac y\mu$. Observe that, for $y\in[\mu x_1, \mu x_n]$, we have $x\in[  x_1  ,   x_n ]$. Therefore,
\begin{align*}
\E[y - \mu U | \mu U < y] = \mu\E[x-U|U<x]\leqslant  \mu  (x_1 -x_n)  .
\end{align*}
So,
\begin{align*}
 \pi^-(y;\mu,\sigma)&\leqslant \mu(x_n-x_1).
\end{align*}
We may also upper bound $ \pi^+(y;\mu,\sigma)$ with the same value. By taking the expectation, we have
\[
\pi(y;\mu,\sigma) \leqslant \mu(x_n-x_1).
\]
\end{proof}

The next result is a key technical component of our analysis. It lower bounds the largest eigenvalue of the neural tangent kernel by $2/m$ times $\textnormal{TV}_{\pi}^{(1)}(s_\theta)$. The proof technique is inspired by \citet[][Lemma 4]{mulayoff2021implicit}. 
\begin{proposition}
\label{prop:top-eigenvalue-second-derivative}
For any $s_\theta \in \mathcal{S}$,
\[
\lambda_{\max}\Big(\frac1n\sum_{i=1}^n\E_{Y\sim\mathcal{N}(\mu x_i,\sigma^2)}[(\nabla_{\theta} s_{\theta}(Y))(\nabla_{\theta} s_{\theta}(Y))^\top]
\geqslant \frac2m\int_\R|s_\theta''(y)| \pi(y;\mu,\sigma)dy.
\]
\end{proposition}
\begin{proof}
We start by rewriting the matrix on the left-hand side of the inequality. We have
\begin{align*}
&\frac1n\sum_{i=1}^n\E_{Y\sim\mathcal{N}(\mu x_i,\sigma^2)}\big[(\nabla_\theta s_\theta(Y))(\nabla_\theta s_\theta(Y))^\top\big]\\
& \quad =\frac1n\sum_{i=1}^n \E_{
\xi\sim\mathcal{N}(0, \sigma^2 )}\big[(\nabla_\theta s_\theta(\xi+ \mu x_i))(\nabla_\theta s_\theta(\xi+ \mu x_i))^\top\big].
\end{align*}
Let $\Phi(\xi)=[\nabla_\theta s_\theta(\xi+  \mu x_1), \dots, \nabla_\theta s_\theta(\xi+  \mu x_n)] \in \R^{3m \times n}$, and notice that the left-hand side is $\frac1n\E\Phi(\xi)\Phi(\xi)^\top$. Denoting $\mathbb{S}^{d-1}$ the sphere of $\R^d$, we may then deduce that 
\begin{align*}
\lambda_{\max}\Big(\frac1n\sum_{i=1}^n\E_{Y\sim\mathcal{N}(\mu x_i,\sigma^2)}\big[(\nabla_\theta s_\theta(Y))(\nabla_\theta s_\theta(Y))^\top\big]\Big)&= \frac1n\max_{v\in\mathbb{S}^{2m-1}}\E _{\xi\sim\mathcal{N}(0,\sigma^2 )}[v^\top\Phi(\xi)\Phi(\xi)^\top v],\\
&=\frac1n\max_{v\in\mathbb{S}^{2m-1}}\E _{\xi\sim\mathcal{N}(0,\sigma^2)}[\|\Phi(\xi)^\top v\|^2],\\
&=\frac1n\max_{u\in\mathbb{S}^{n-1}}\E _{\xi\sim\mathcal{N}(0, \sigma^2)}[\|\Phi(\xi)u\|^2],\\
&\geqslant \frac1n \E _{\xi\sim\mathcal{N}(0, \sigma^2 )}[\frac1n\|\Phi(\xi)\mathbf{1}\|^2],
\end{align*}
where $\mathbf{1}=(1, 1, \dots, 1)^\top\in\R^n$, so that $\frac{1}{\sqrt{n}}\mathbf{1}$ is a unit vector. To this aim, let us lower bound $\frac1{n^2}\|\Phi(\xi)\mathbf{1}\|^2$ for a fixed $\xi \in \R^n$, and then take the expectation with respect to $\xi$. First, let $I_{i, \ell} = \delta(w_\ell^{(1)}(\xi+\mu x_i)+b_\ell>0)$, where $\delta(\cdot)$ equals 1 if the condition is satisfied otherwise equals 0, and $\mathbb{I}_i\in\R^m$ the vector whose $\ell$-th entry is 
$I_{i, \ell}$. We may then calculate the gradient of $s_\theta$ as follows:
\[
\nabla_\theta s_\theta(y)
=
\left(\begin{array}{c}
\nabla_{\mathbf{w}^{(2)}}s_\theta(y) \\ \nabla_{\mathbf{b}}s_\theta(y)
\end{array}\right)
=
\left(\begin{array}{c} \frac1m\big(y\mathbf{w}^{(1)}+\mathbf{b}\big)\odot\mathbb{I}_i \\ 
\frac1m\mathbf{w}^{(2)}\odot\mathbb{I}_i
\end{array}\right).
\]
Then, by the inequality $a^2+b^2 \geqslant 2 ab$,
\begin{align*}
\frac1{n^2}\|\Phi(\xi)\mathbf{1}\|^2
&=\frac1{m^2n^2}\sum_{\ell=1}^m\Big[\Big(\sum_{i=1}^n\phi\big(w_\ell^{(1)}(\xi+\mu x_i)+b_\ell\big)\Big)^2 + \Big(\sum_{i=1}^n I_{i, \ell}w_\ell^{(2)}\Big)^2\Big]\\
&\geqslant \frac2{m^2n^2}\sum_{\ell=1}^m \big|w_\ell^{(2)} \big|\times\Big|\sum_{i=1}^n\phi\big(w_\ell^{(1)}(\xi+\mu x_i)+b_\ell\big)\Big|\times\Big(\sum_{i=1}^nI_{i, \ell}\Big).
\end{align*}
Define $C_\ell=\{y \in \R, w_\ell^{(1)}y+b_\ell>0\}$ and $n_\ell=\sum_{i=1}^n I_{i, \ell} = |C_\ell|
$. 
Recall that $U$ denotes a random variable drawn uniformly from the training data $\{x_j\}_{1\leqslant j\leqslant n}$.
Then, we have
\begin{align*}
\frac1{n^2}\|\Phi(\xi)&\mathbf{1}\|^2\geqslant\frac{2}{m^2}\sum_{\ell=1}^m \Big(\frac{n_\ell}{n}\Big)^2|w_\ell^{(2)}|\times\Big|\frac1{n_\ell}\sum_{(\xi+\mu U)\in C_\ell}w_\ell^{(1)}(\xi+\mu U)+b_\ell\Big|\\
&=\frac{2}{m^2}\sum_{\ell=1}^m \PP((\xi+\mu U)\in C_\ell)^2\times |w_\ell^{(2)}|\times \Big|\E\Big[w_\ell^{(1)}(\xi+\mu U)+b_\ell\mid (\xi+\mu U)\in C_\ell\Big]\Big| ,
\end{align*}
where the probability and expectation are taken with respect to $U$.
Next, define $\tau_\ell = - b_\ell / w_\ell^{(1)}$.
Since $|w_\ell^{(1)}|=1$, we may then rewrite
\begin{align*}
\E\big[w_\ell^{(1)}(\xi+\mu U)+b_\ell\big|(\xi+\mu U)\in C_\ell\big]
= \E\big[|\xi+\mu U-\tau_\ell| \big| (\xi+\mu U)\in C_\ell\big].
\end{align*}
Notice that the set $C_\ell$ is an interval with one end at $\pm\infty$ and another at $\tau_\ell$, depending on the sign of $w_\ell^{(1)}$. If $w_\ell^{(1)} = 1$, we have $C_\ell = (-b_\ell,\infty)=(\tau_\ell,\infty)$, and
\begin{align*}
\PP((\xi+\mu U)\in C_\ell)^2 \big|\E\big[\xi+\mu U-\tau_\ell &\big| \xi+\mu U)\in C_\ell\big]\big| \\
&= \PP(\mu U > \tau_\ell-\xi)^2\times \E\big[\mu U-\tau_\ell+\xi\big| \mu U > \tau_\ell-\xi\big] \\
&= \pi^+(\tau_\ell-\xi) ,
\end{align*}
where we recall that $\pi^+$ is defined in Section \ref{sec:regularity-empirical-score}. Likewise, if $w_\ell^{(1)} = -1$, we obtain
\[
\PP((\xi+\mu U)\in C_\ell)^2 \big|\E\big[\xi+\mu U-\tau_\ell \big| \xi+\mu U)\in C_\ell\big]\big| = \pi^-(\tau_\ell-\xi) .
\]
All in all, we get, for all $\ell \in [1, m]$, 
\[
\PP((\xi+\mu U)\in C_\ell)^2\times \big|\E\big[(\xi+\mu U)-\tau_\ell\big| (\xi+\mu U)\in C_\ell\big]\big|\geqslant \min\{\pi^+(\tau_\ell-\xi; \mu,\sigma), \pi^-(\tau_\ell-\xi;\mu, \sigma)\}.
\]
Thus, since $s_\theta''(y) = \frac1m\sum_{\ell=1}^mw_\ell^{(1)}w_\ell^{(2)}\delta(y=\tau_\ell)$, we are led to
\begin{align*}
\frac1{n^2}\|\Phi(\xi)\mathbf{1}\|^2&\geqslant \frac2{m^2}\sum_{\ell=1}^m|w_\ell^{(2)}|\min\{\pi^+(\tau_\ell-\xi; \mu, \sigma), \pi^-(\tau_\ell-\xi;\mu,\sigma)\},\\
&\geqslant \frac2{m}\int_\R|s_\theta''(y)|\min\{\pi^+(y-\xi; \mu,\sigma), \pi^-(y-\xi; \mu, \sigma)\}dy.
\end{align*}
We may then take the expectation value on both side and conclude that
\[
\lambda_{\max}\Big(\frac1n\sum_{i=1}^n\E_{Y\sim\mathcal{N}(\mu x_i, \sigma^2)}\big[(\nabla_\theta s_\theta(Y))(\nabla_\theta s_\theta(Y))^\top\big]\Big) \geqslant \frac2m\int_\R |s_\theta''(y)|\pi(y; \mu, \sigma)dy.
\]
\end{proof}

\section{Proofs of the results of the main text} \label{apx:proofs}

\subsection{Proof of Proposition \ref{prop:variance-bound}}
We start by proving that
\[
|\E[W(y; \mu, \sigma)]- x_i| \leqslant n\Delta e^{-\frac{\mu^2\Delta^2}{4\sigma^2}},
\]
when $|y - \mu x_i|\leqslant \frac\mu4\Delta$. Let $x = \frac{y}{\mu}$, then $|y-\mu x_i| \leqslant\frac{\mu}{4}\Delta$ implies that $|x-x_i|\leqslant\frac\Delta4$. We show that
\begin{equation}
\label{eq:local-prop-var-bound}
\E[W(\mu x; \mu,\sigma)] - x_i\leqslant n\Delta e^{-\frac{\mu^2\Delta^2}{4\sigma^2}},
\end{equation}
and the other side can be deduced similarly. Observe that
\begin{align*}
\E[W(\mu x; \mu,\sigma)] - x_i  &= \sum_{i'=1}^n (x_{i'} - x_i )\alpha_{i'}( \mu x;\mu, \sigma),\\
&\leqslant \frac{\sum_{i'=i+1}^n(x_{i'}-x_i )e^{-\frac{\mu^2(x_{i'}-x)^2}{2\sigma^2}}}{\sum_{i'=1}^ne^{-\frac{\mu^2(x_{i'}-x)^2}{2\sigma^2}}},\\
&=\frac{\sum_{i'=i+1}^n(x_{i'}-x_i )e^{-\frac{\mu^2\big((x_{i'}-x)^2-(x_i -x)^2\big)}{2\sigma^2}}}{\sum_{i'=1}^n e^{-\frac{  \mu^2\big((x_{i'}-x)^2-(x_i -x)^2\big)}{2\sigma^2}}},\\
&\leqslant \sum_{i'=i+1}^{n} (x_{i'}-x_i )e^{-\frac{\mu^2(x_{i'}-x_i )(x_{i'}+x_i -2x)}{2\sigma^2}},
\end{align*}
where, in the last inequality, we used the fact that when $i'=i$, $e^{-  
\frac{\mu^2\big((x_{i'}-x)^2-(x_i -x)^2\big)}{2\sigma^2}}=1$, and so the denominator is larger than 1. Next, for $i'>i$, with the condition 
\[
-\frac{x_{i'}-x_i }{4}\leqslant-\frac{\Delta}{4}\leqslant x_i -x\leqslant0, 
\]
which implies that $x_{i'}-x\geqslant \frac{3}{4}(x_{i'}-x_i )$, we have $x_{i'}+x_i -2x\geqslant \frac12(x_{i'}-x_i )\geqslant 0$.
Therefore,
\begin{align}\label{eq:local-bound}
\E[W( \mu x; \mu,\sigma)] -x_i  &\leqslant \sum_{i'=i+1}^{n}(x_{i'}-x_i )e^{- \frac{\mu^2 (x_i -x_{i'})^2}{4\sigma^2}}\\
&\leqslant n\Delta e^{-\frac{\mu^2\Delta^2}{4\sigma^2}}, \nonumber
\end{align}
where we apply Lemma \ref{lem:local-function-bound} with $k=1$ and $A=\frac{\mu^2}{2\sigma^2}$ the function $x\mapsto xe^{-\frac{\mu^2x^2}{4\sigma^2}}$ is strictly decreasing on $[\sqrt2\frac{\sigma}{\mu},\infty)$, and $x_{i'} - x_{i} \geqslant \Delta \geqslant 2\frac{\sigma}{\mu} \geqslant \sqrt2\frac{\sigma}{\mu}$. Hence, we deduce \eqref{eq:local-prop-var-bound}. We now turn to the bounds of the variance of $W$.
\paragraph{Lower bounding $\V[W(m_i; \mu,\sigma)]$.} We start by observing that
\[
\alpha_i(m_i;\mu,\sigma) \geqslant \frac1n \quad \mbox{and } \alpha_{i+1}(m_i;\mu,\sigma) \geqslant \frac1n,
\]
since $\alpha_i(m_i;\mu,\sigma) = \frac{e^{-\frac{\mu^2(m_i-  x_i )^2}{2\sigma^2} }}{\sum_{j=1}^n e^{-\frac{\mu^2(m_i-  x_i )^2}{2\sigma^2}}}$ and, for $1\leqslant j\leqslant n$, 
\[
e^{-\frac{\mu^2(m_i-  x_j)^2}{2\sigma^2}}\leqslant e^{-\frac{\mu^2(m_i-  x_i )^2}{2\sigma^2}}.
\]
We may then lower bound $\V W(m_i)$ as follows:
\begin{align*}
\V[W(m_i;&\mu,\sigma)]\\
&\geqslant \alpha_i(m_i;\mu,\sigma)(x_i -\E [W(m_i;\mu,\sigma)])^2+\alpha_{i+1}(m_i;\mu,\sigma)(x_{i+1}-\E[W(m_i;\mu,\sigma)])^2\\
&\geqslant\frac1n \big((x_i -\E[W(m_i;\mu,\sigma)])^2+(x_{i+1}-\E[W(m_i;\mu,\sigma)])^2\big)\\
&\geqslant \frac1{2n}(x_i -x_{i+1})^2,
\end{align*}
where the last inequality is derived by applying the Cauchy-Schwarz inequality.
\paragraph{Upper bounding $\V[W(y;\mu,\sigma)]$.} Let $x = \frac{y}{\mu}$. The condition $|y - \mu x_i|\leqslant \frac{\mu\Delta}{4}$ can be rewritten in terms of $x$ by $|x - x_i|\leqslant \frac\Delta4$. For $1\leqslant i'\leqslant n$, we have the following bound, which is a consequence of the Cauchy-Schwarz inequality and the previous paragraph:
\[
(x_{i'}-\E[W(\mu x;\mu,\sigma)])^2\leqslant 2(x_{i'}-x_i )^2+2(\E[W( \mu x;\mu,\sigma)]-x_i )^2\leqslant 2(x_{i'}-x_i )^2+2n^2\Delta^2e^{-\frac{\mu^2\Delta^2}{2\sigma^2}}.
\]
Thus,
\begin{align*}
\V[W(\mu x;& \mu,\sigma)]\\
&=\alpha_i(\mu x;\mu,\sigma)(x_i -\E[W( \mu x;\mu,\sigma)])^2 + \sum_{i'\neq i}\alpha_{i'}(  \mu x;\mu,\sigma)(x_{i'}-\E[ W(\mu x;\mu,\sigma)])^2\\
&\leqslant \alpha_i(\mu x;\mu,\sigma)n^2\Delta^2e^{-\frac{\mu^2\Delta^2}{2\sigma^2}}+2\sum_{i'\neq i} \alpha_{i'}(\mu x;\mu,\sigma)n^2\Delta^2e^{-\frac{\mu^2\Delta^2}{2\sigma^2}}\\
& \quad +2\sum_{i'\neq i}(x_{i'}-x_i )^2\alpha_{i'}(\mu x;\mu,\sigma)\\
&\leqslant 2n^2\Delta^2e^{-\frac{\mu^2\Delta^2}{2\sigma^2}}+2\sum_{i'\neq i}(x_{i'}-x_i )^2e^{-\frac{\mu^2(x_{i'}-x_i )^2}{4\sigma^2}}
\end{align*}
where we apply the same argument as in \eqref{eq:local-bound} to bound the last term. Next, applying Lemma~\ref{lem:local-function-bound} with $k=2$ and $A=\frac{\mu^2}{2\sigma^2 }$, we have that $x\mapsto x^2 e^{-\frac{\mu^2x^2}{4\sigma^2}}$ is strictly decreasing on $[2\frac{\sigma}{\mu},\infty)$. Since $x_{i'} - x_{i} \geqslant \Delta \geqslant 2\frac{\sigma}{\mu}$, we obtain
\begin{align*}
\V[W(\mu x; \mu,\sigma)] &\leqslant 2n^2\Delta^2e^{-\frac{\mu^2\Delta^2}{2\sigma^2}}+2n\Delta^2e^{-\frac{\mu^2\Delta^2}{4\sigma^2}} \leqslant 4n^2\Delta^2e^{-\frac{\mu^2\Delta^2}{4\sigma^2}}.
\end{align*}
This concludes the proof.
\subsection{Proof of Corollary \ref{cor:regularity-empirical-optimal-score}}
The first inequality of the Corollary unfolds from Proposition \ref{prop:variance-bound} and equation \eqref{eq:sstarprimer}.
By applying Proposition~\ref{prop:variance-bound} we have that 
\begin{equation}
\label{eq:bound-optimal-local-linearity}
\big|s^\star( y;\mu,\sigma) - \frac{1}{\sigma^2}(\mu x_i -y)\big| \leqslant \frac{n\mu\Delta}{\sigma^2} e^{-\frac{\mu^2\Delta^2}{4\sigma^2}}
\end{equation}
We now focus on upper bounding the loss of the empirical optimal score. We have
\begin{align*}
&\cR_n(s^\star) = \frac1n\sum_{i=1}^n \E_{Y\sim\cN(\mu x_i, \sigma^2)}[(s^\star(Y; \mu, \sigma)+\frac1{\sigma^2}(Y-\mu x_i))^2]\\
&=\frac1n\sum_{i=1}^n\frac{1}{\sqrt{2\pi\sigma^2}}\big(\!\int_{\mu x_i-\frac{\mu}{4}\Delta}^{\mu x_i+\frac{\mu}{4}\Delta}+\int_{-\infty}^{\mu x_i-\frac{\mu}{4}\Delta}+\int_{\mu x_i+\frac{\mu}{4}\Delta}^\infty\!\big)(s^\star(y;\mu,\sigma)+\frac1{\sigma^2}(y-\mu x_i))^2e^{-\frac{(y-\mu x_i)^2}{2\sigma^2}}dy \\
&\leqslant \frac{1}{n}\sum_{i=1}^n \frac{n^2\mu^2\Delta^2}{\sigma^4}e^{-\frac{\mu^2\Delta^2}{2\sigma^2}}+\frac1n\sum_{i=1}^n \frac{\mu^2}{\sigma^4}(x_n-x_1)^2\PP_{Y\sim\cN(\mu x_i, \sigma^2)}(|Y-\mu x_i|>\frac{\mu}{4}\Delta),
\end{align*}
where for the first term in the last inequality we use \eqref{eq:bound-optimal-local-linearity} and for the second term, we use the fact that 
\[
|s^\star(y;\mu,\sigma)+\frac{1}{\sigma^2}(y-\mu x_i)| = \frac{\mu}{\sigma^2}|x_i-\E[W(y; \mu, \sigma)]|.
\]
Since both $W(y;\mu,\sigma)$ and $x_i$ only take value between $[x_1, x_n]$, we have that
\[
|s^\star(y;\mu,\sigma)+\frac{1}{\sigma^2}(y-\mu x_i)|\leqslant\frac{\mu}{\sigma^2}(x_n-x_1).
\]
We then obtain by applying a tail bound of the Gaussian distribution \citep{gordon1941values} that
\[
\cR_n(s^\star) \leqslant \frac{n^2\mu^2\Delta^2}{\sigma^4}e^{-\frac{\mu^2\Delta^2}{2\sigma^2}} + \frac{2 \mu^2 (x_n-x_1)^2}{\sigma^4} e^{-\frac{\mu^2 \Delta^2}{32 \sigma^2}} .
\]
Observe that $n \Delta \leqslant 2(x_n - x_1)$. Thus we get
\[
\cR_n(s^\star) \leqslant \frac{4 \mu^2 (x_n-x_1)^2}{\sigma^4} e^{-\frac{\mu^2 \Delta^2}{32 \sigma^2}} ,
\]
which concludes the proof.

\subsection{Proof of Lemma \ref{lem:twice-differentiability}}
Our approach to proving that $\cR_n(\theta)$ is twice differentiable involves explicitly computing its Hessian, as the resulting expression will be instrumental in the subsequent proof. First note, since $\phi$ is differentiable almost everywhere, that by the dominated convergence theorem,
\[
\nabla_\theta \cR_n(\theta) = \frac2n\sum_{i=1}^n\E_{Y\sim\mathcal{N}(\mu x_i, \sigma^2)}\big[\big(s_\theta(Y)+\frac1{\sigma^2}(Y-\mu x_i)\big)\nabla_\theta s_\theta(Y)\big],
\]
with
\begin{equation}    \label{eq:tech-proof-lemma5}
\nabla_\theta s_\theta(y)
=
\left(\begin{array}{c}
\nabla_{\mathbf{w}^{(2)}}s_\theta(y) \\
\nabla_{\mathbf{b}}s_\theta(y)
\end{array}\right)
=
\left(\begin{array}{c} 
\frac1m \phi \big(y\mathbf{w}^{(1)}+\mathbf{b}\big) \\ 
\frac1m\mathbf{w}^{(2)}\odot \mathbf{1}_{y\mathbf{w}^{(1)}+\mathbf{b} \geqslant 0}
\end{array} \right).    
\end{equation}
In the expression above and throughout the remainder of this proof, bold symbols represent vectors in $\R^m$, where each entry corresponds to a neuron. For example, $\mathbf{b} = (b_1, \dots, b_m)$. We will also make use of the notation $\diag(v)$ to denote the square matrix whose diagonal is the vector $v$.
Next, using again that $\phi$ is differentiable almost everywhere to differentiate a second time the first part of the expression, we obtain
\begin{align*}
\nabla_\theta^2\cR_n(\theta) 
&= \frac{2}{n}\sum_{i=1}^n \E_{Y\sim\mathcal{N}(\mu x_i, \sigma^2)}\big[(\nabla_\theta s_\theta(Y))(\nabla_\theta s_\theta(Y))^\top\big]\\ 
&\quad+  \frac2n\sum_{i=1}^n \nabla_\nu\E_{Y\sim\mathcal{N}(\mu x_i, \sigma^2)}\big[\big(s_\theta(Y)+\frac1{\sigma^2 }(Y -  \mu x_i)\big)(\nabla_\nu s_\nu(Y))^\top\big] \Big|_{\nu=\theta} .
\end{align*}
The notation in the second sum means that we are only considering the derivative with respect to the parameters appearing in the gradient term $\nabla_\nu s_\nu(Y)$, and not in the loss term $s_\theta(Y)+\frac1{\sigma^2 }(Y -  \mu x_i)$.
Fixing $i \in \{1, \dots, n\}$, and denoting by $M \in \R^{2m \times 2m}$ the matrix inside the second sum, we observe that $M$ is a block matrix, each block corresponding to differentiating with respect to either $\mathbf{w}$ or $\mathbf{b}$. More precisely,
\[
M
=
\left(
\begin{array}{cc}
M_{\mathbf{ww}} & M_{\mathbf{wb}} \\ 
M_{\mathbf{bw}} & M_{\mathbf{bb}}
\end{array}
\right),
\]
where three blocks $M_{\mathbf{ww}}$, $M_{\mathbf{wb}}$, and $M_{\mathbf{bw}}$ are straightforward to compute, i.e.,
\begin{align*}
M_{\mathbf{ww}} &= \textbf{0}_{m \times m}, \\
M_{\mathbf{wb}} &= M_{\mathbf{bw}} = \frac{1}{m} \diag \E_{Y\sim\mathcal{N}(\mu x_i, \sigma^2)}\big[\big(s_\theta(Y)+\frac1{\sigma^2 }(Y -  \mu x_i)\big) \mathbf{1}_{Y\mathbf{w}^{(1)}+\mathbf{b} \geqslant 0}\big], 
\end{align*}
where $\textbf{0}_{m \times m}$ is the null matrix in $\R^{m \times m}$. Computing the last block $M_{\mathbf{bb}}$ is more delicate, because of the term $\mathbf{1}_{y\mathbf{w}^{(1)}+\mathbf{b} \geqslant 0}$ appearing in the gradient \eqref{eq:tech-proof-lemma5} of $\cR_n$ with respect to $\mathbf{b}$. Observe that $M_{\mathbf{bb}}$ is a diagonal matrix, whose $\ell$-th diagonal element is
\begin{align*}
(M_{\mathbf{bb}})_{\ell \ell} 
&= \partial_{b} \E_{Y\sim\mathcal{N}(\mu x_i, \sigma^2)}\big[\big(s_\theta(Y)+\frac1{\sigma^2 }(Y -  \mu x_i)\big) \frac{1}{m} w^{(2)}_\ell \mathbf{1}_{Y w^{(1)}_\ell + b \geqslant 0} \big] \Big|_{b=b_\ell}.
\end{align*}
Here, this notation indicates once again that we take the derivative only with respect to the term~$b$ appearing in the indicator function, and not with respect to~$s_\theta(Y)$. To compute this quantity, we consider two cases based on the value of $w^{(1)}_\ell = \pm 1$. If $w^{(1)}_\ell = 1$, we have 
\begin{align*}
(M_{\mathbf{bb}})_{\ell \ell} 
&= \frac{1}{m} w^{(2)}_\ell \Big(\partial_{b} \int_{-b}^\infty \big(s_\theta(y)+\frac1{\sigma^2 }(y -  \mu x_i)\big) \frac{1}{\sqrt{2\pi} \sigma} e^{-\frac{(y-\mu x_i)^2}{2\sigma^2}} dy \Big) \Big|_{b=b_\ell} \\
&= - \frac{1}{\sqrt{2\pi} m \sigma} w^{(2)}_\ell \big(s_\theta(-b_\ell)+\frac1{\sigma^2 }(- b_\ell -  \mu x_i)\big) e^{-\frac{(-b_\ell-\mu x_i)^2}{2\sigma^2}} .
\end{align*}
A similar computation shows that, if $w^{(1)}_\ell = -1$,
\[
(M_{\mathbf{bb}})_{\ell \ell} = \frac{1}{\sqrt{2\pi} m \sigma} w^{(2)}_\ell \big(s_\theta(b_\ell)+\frac1{\sigma^2 }(b_\ell -  \mu x_i)\big) e^{-\frac{(b_\ell-\mu x_i)^2}{2\sigma^2}} .
\]
Letting $\tau_\ell = -b_\ell / w^{(1)}_\ell$, we can summarize both cases in a single formula:
\[
(M_{\mathbf{bb}})_{\ell \ell} = - \frac{1}{\sqrt{2\pi} m \sigma} w^{(2)}_\ell w^{(1)}_\ell \big(s_\theta(\tau_\ell)+\frac1{\sigma^2 }(\tau_\ell -  \mu x_i)\big) e^{-\frac{(\tau_\ell-\mu x_i)^2}{2\sigma^2}} .
\]
All in all, we are led to
\begin{align}   
\nabla_\theta^2\cR_n(\theta) 
&= \frac{2}{n}\sum_{i=1}^n \E_{Y\sim\mathcal{N}(\mu x_i, \sigma^2)}\big[(\nabla_\theta s_\theta(Y))(\nabla_\theta s_\theta(Y))^\top\big] \nonumber \\ 
&\quad+  \frac{2}{mn} \sum_{i=1}^n \E_{Y\sim\mathcal{N}(\mu x_i, \sigma^2)}\big[\big(s_\theta(Y)+\frac1{\sigma^2 }(Y -  \mu x_i)\big)H(Y)\big] \nonumber  \\
&\quad- \frac{\sqrt{2}}{\sqrt{\pi} m n \sigma} \sum_{i=1}^n \diag (\mathbf{0}_{m}, \mathbf{w^{(2)}} \mathbf{w^{(1)}} \big(s_\theta(\boldsymbol{\tau})+\frac1{\sigma^2 }(\boldsymbol{\tau} -  \mu x_i)\big) e^{-\frac{(\boldsymbol{\tau} -\mu x_i)^2}{2\sigma^2}}) , \label{eq:formula-hessian}
\end{align}
where $\mathbf{0}_{m}$ denotes the null vector in $\R^m$ and
\begin{equation}    \label{eq:formula-H}
H(y) = \left(
\begin{array}{cc}
\mathbf{0}_{m \times m} & \diag \mathbf{1}_{y\mathbf{w}^{(1)}+\mathbf{b} \geqslant 0} \\ 
\diag \mathbf{1}_{y\mathbf{w}^{(1)}+\mathbf{b} \geqslant 0} & \mathbf{0}_{m \times m}
\end{array}
\right) .    
\end{equation}
This concludes the proof.

\subsection{Proof of Proposition \ref{prop:second-derivative-upper-bound}}
We start from the formula for the Hessian of the loss provided by \eqref{eq:formula-hessian}. Let $D$ be the diagonal matrix in the third term of the Hessian, and let $v$ be a unit eigenvector of the first term  with respect to its largest eigenvalue. Recalling that, for any matrix $M$, $\lambda_{\max}(M) \geqslant v^\top M v$ with equality if $v$ is an eigenvector of $M$ with eigenvalue $\lambda_{\max}(M)$, we then have
\begin{align*}
\lambda_{\max}(\nabla_\theta^2\cR_n(\theta))&\geqslant v^\top\nabla_\theta^2\cR_n(\theta)v\\
&=\lambda_{\max}\Big(\frac{2}{mn}\sum_{i=1}^n \E_{Y\sim\mathcal{N}(\mu x_i, \sigma^2)}\Big[(\nabla_\theta s_\theta(Y))(\nabla_\theta s_\theta(Y))^\top\Big]\Big)\\
&\qquad+\frac{2}{mn}\sum_{i=1}^n\E_{Y\sim\mathcal{N}(\mu x_i, \sigma^2)}\Big[\big(s_\theta(Y)+\frac1{\sigma^2}(Y-\mu x_i)\big)v^\top H(Y)v\Big]\\
&\qquad +v^\top D v.
\end{align*}
The first term is lower bounded using Proposition \ref{prop:top-eigenvalue-second-derivative}. Thus, rearranging the terms, we obtain
\begin{align}  \label{eq:tech-proof-6}
\begin{split}
&\lambda_{\max}(\nabla_\theta^2\cR_n(\theta))+\Big|\frac{2}{mn}\sum_{i=1}^n\E_{Y\sim\mathcal{N}(\mu x_i, \sigma^2)}\Big[\big(s_\theta(Y)+\frac1{\sigma^2}(Y-\mu x_i)\big)v^\top H(Y)v\Big]\Big| + |v^\top D v| \\
&\quad\geqslant \frac4m \textnormal{TV}_{\pi}^{(1)}(s_\theta) .    
\end{split}
\end{align}
We now bound the second and third term on the left-hand side of the inequality above. 
As for the second term, using the Cauchy-Schwarz inequality twice,
\begin{align*}
&\Big|\frac{2}{mn}\sum_{i=1}^n\E_{Y\sim\mathcal{N}(\mu x_i, \sigma^2)}\Big[\big(s_\theta(Y)+\frac1{\sigma^2}(Y-\mu x_i)\big)v^\top H(Y)v\Big]\Big| \\
&\leqslant\frac{2}{mn}\sum_{i=1}^n\sqrt{\E_{Y\sim\mathcal{N}(\mu x_i, \sigma^2)}\Big[\big(s_\theta(Y)+\frac1{\sigma^2}(Y-\mu x_i)\big)^2\Big]}\sqrt{\E_{Y\sim\mathcal{N}(\mu x_i, \sigma^2)}\big(v^\top H(Y)v\big)^2} \\
&\leqslant \frac{2}{m}\sqrt{\frac1n\sum_{i=1}^n\E_{Y\sim\mathcal{N}(\mu x_i, \sigma^2)}\Big[\big(s_\theta(Y)+\frac1{\sigma^2}(Y-\mu x_i)\big)^2\Big]} \sqrt{\frac1n\sum_{i=1}^n\E_{Y\sim\mathcal{N}(\mu x_i, \sigma^2)}\big(v^\top H(Y) v\big)^2}  \\
&\leqslant\frac{2}{m}\sqrt{\cR_n(\theta)}\sqrt{\frac1n\sum_{i=1}^n\E_{Y\sim\mathcal{N}(\mu x_i, \sigma^2)}\big[\lambda_{\max}(H(Y))^2\big]}. 
\end{align*}
By inspecting formula \eqref{eq:formula-H} for $H(y)$, we easily see that, for any $y \in \R$, $\lambda_{\max}(H(y)) \leqslant 1$. So,
\[
\Big|\frac{2}{mn}\sum_{i=1}^n\E_{Y\sim\mathcal{N}(\mu x_i, \sigma^2)}\Big[\big(s_\theta(Y)+\frac1{\sigma^2}(Y-\mu x_i)\big)v^\top H(Y)v\Big]\Big| \leqslant \frac{2}{m} \sqrt{\cR_n(\theta)}.
\]
We now proceed to bound the term $|v^\top D v|$ in \eqref{eq:tech-proof-6}, where we recall that $\tau_\ell = -b_\ell / w^{(1)}_\ell$. Since $D$ is a diagonal matrix, we have
\begin{align*}
|v^\top D v| &\leqslant \max_{1 \leqslant \ell \leqslant m} |D_{\ell \ell}| \nonumber\\
&= \max_{1 \leqslant \ell \leqslant m} \frac{\sqrt{2}}{\sqrt{\pi} m n \sigma} \Big| \sum_{i=1}^n w^{(2)}_\ell w^{(1)}_\ell \big(s_\theta(\tau_\ell)+\frac1{\sigma^2 }(\tau_\ell -  \mu x_i)\big) e^{-\frac{(\tau_\ell -\mu x_i)^2}{2\sigma^2}} \Big| \nonumber\\
&\leqslant \max_{1 \leqslant \ell \leqslant m} \frac{\sqrt{2}}{\sqrt{\pi} m n \sigma} \sum_{i=1}^n |w^{(2)}_\ell|  \times \big|s_\theta(\tau_\ell)+\frac1{\sigma^2 }(\tau_\ell -  \mu x_i)\big| e^{-\frac{(\tau_\ell -\mu x_i)^2}{2\sigma^2}} \nonumber\\
&\leqslant \max_{1 \leqslant \ell \leqslant m, 1 \leqslant i \leqslant n} \frac{\sqrt{2} A}{\sqrt{\pi} m \sigma}  \big|s_\theta(\tau_\ell)+\frac1{\sigma^2 }(\tau_\ell -  \mu x_i)\big| e^{-\frac{(\tau_\ell -\mu x_i)^2}{2\sigma^2}} . 
\end{align*}
Thus, so far, we have proved that
\begin{align}   \label{eq:tech-proof-prop6-2}
\begin{split}
\frac4m \textnormal{TV}_{\pi}^{(1)}(s_\theta) &\leqslant \lambda_{\max}(\nabla_\theta^2\cR_n(\theta))+\frac{2}m \sqrt{\cR_n(\theta)} \\
&\quad + \frac{ \sqrt{2} A}{\sqrt{\pi} m \sigma} \max_{1 \leqslant \ell \leqslant m, 1 \leqslant i \leqslant n} \big|s_\theta(\tau_\ell)+\frac1{\sigma^2 }(\tau_\ell -  \mu x_i)\big| e^{-\frac{(\tau_\ell -\mu x_i)^2}{2\sigma^2}} .    
\end{split}
\end{align}
Now, let $f_\theta(y) = s_\theta(y)+\frac1{\sigma^2 }(y -  \mu x_i)$. Our proof strategy consists in deriving a bound on $|f_\theta(\tau_\ell)| e^{-\frac{(\tau_\ell - \mu x_i)^2}{2 \sigma^2}}$ depending on the value of the risk $\mathcal{R}_n(\theta)$, and valid for all $\ell \in \{1, \dots, m\}$ and all $i \in \{1, \dots, n\}$. To this aim, first note that %
\begin{align*}
n \mathcal{R}_n(\theta) \geqslant \mathbb{E}_{Y \sim \cN(\mu xi, \sigma^2)} \Big[\big(s_\theta(Y)+\frac1{\sigma^2}(Y-\mu x_i)\big)^2\Big] = \mathbb{E}_{Y \sim \cN(\mu xi, \sigma^2)} \big[f_\theta(Y)^2\big].
\end{align*}
We observe that $f_\theta(y)$ is Lipschitz continuous with Lipschitz constant at most $A + \frac{1}{\sigma^2} \leqslant 2A$, where the inequality holds since we assumed that $A \geqslant C_n / \sigma^6 \geqslant 1/\sigma^2$. Thus, for $y \in \R$,
\[
|f_\theta(y)| \geqslant |f_\theta(\tau_\ell)| - 2A |y - \tau_\ell| .
\]
So, if $|y - \tau_\ell| \leqslant \frac{|f_\theta(\tau_\ell)|}{4A}$,
\[
|f_\theta(y)| \geqslant \frac{|f_\theta(\tau_\ell)|}{2}.
\]
Therefore,
\begin{align*}
n \mathcal{R}_n(\theta) &\geqslant \frac{1}{\sqrt{2\pi} \sigma} \int_{\R} f_\theta(y)^2 e^{-\frac{(y -\mu x_i)^2}{2\sigma^2}} dy \\
&\geqslant \frac{1}{\sqrt{2\pi} \sigma} \int_{\tau_\ell - |f_\theta(\tau_\ell)|/4A}^{\tau_\ell + |f_\theta(\tau_\ell)|/4A} f_\theta(y)^2 e^{-\frac{(y -\mu x_i)^2}{2\sigma^2}} dy \\
&\geqslant \frac{1}{\sqrt{2\pi} \sigma} \int_{\tau_\ell - |f_\theta(\tau_\ell)|/4A}^{\tau_\ell + |f_\theta(\tau_\ell)|/4A} \frac{f_\theta(\tau_\ell)^2}{4} e^{-\frac{(y -\mu x_i)^2}{2\sigma^2}} dy \\
&= \frac{f_\theta(\tau_\ell)^2}{4\sqrt{2\pi} \sigma} \int_{\tau_\ell - \mu x_i - |f_\theta(\tau_\ell)|/4A}^{\tau_\ell - \mu x_i + |f_\theta(\tau_\ell)|/4A} e^{-\frac{y^2}{2\sigma^2}} dy \\
&= \frac{f_\theta(\tau_\ell)^2}{4\sqrt{2\pi} \sigma} \int_{|\tau_\ell - \mu x_i| - |f_\theta(\tau_\ell)|/4A}^{|\tau_\ell - \mu x_i| + |f_\theta(\tau_\ell)|/4A} e^{-\frac{y^2}{2\sigma^2}} dy ,
\end{align*}
where the last step follows from the symmetry of the Gaussian distribution around $0$. Denote by $I$ the last integral and $D$ its integration domain.
To lower bound $I$, two cases are considered:
\paragraph{Case 1. $[-\sigma, \sigma]$ is included in $D$.} In this case,
\[
\frac{1}{\sqrt{2\pi} \sigma} I \geqslant \frac{1}{\sqrt{2\pi} \sigma} \int_{-\sigma}^{\sigma} e^{-\frac{y^2}{2\sigma^2}} dy = \frac{1}{\sqrt{2\pi}} \int_{-1}^{1} e^{-\frac{y^2}{2}} dy \geqslant \frac{1}{2} ,
\]
and thus $n \mathcal{R}_n(\theta) \geqslant \frac{f_\theta(\tau_\ell)^2}{8}$.
We conclude, when $[-\sigma, \sigma]$ is included in $D$, that
\[
|f_\theta(\tau_\ell)| e^{-\frac{(\tau_\ell - \mu x_i)^2}{2 \sigma^2}} \leqslant |f_\theta(\tau_\ell)| \leqslant 2 \sqrt{2 n \mathcal{R}_n(\theta)}.
\]
\paragraph{Case 2. $[-\sigma, \sigma]$ is not included in $D$.} Since the absolute value of the upper endpoint of $D$ is larger than the absolute value of its lower endpoint, the condition implies that the lower endpoint of~$D$ is larger than $-\sigma$. Therefore, we have $D \subset \big[-\sigma, |\tau_\ell - \mu x_i|\big]$. Hence, for all $y \in D$,
\[
 e^{-\frac{y^2}{2\sigma^2}} \geqslant \frac{1}{\sqrt{e}} e^{-\frac{(\tau_\ell - \mu x_i)^2}{2 \sigma^2}} .
\]
To see this, notice that, if $y \in D$ and $y<0$, one has $y \geqslant - \sigma$ and so $e^{-\frac{y^2}{2\sigma^2}} \geqslant e^{-\frac{1}{2}}$. On the other hand, if $y \geqslant 0$, then $y \leqslant |\tau_\ell - \mu x_i|$ gives $e^{-\frac{y^2}{2\sigma^2}} \geqslant e^{-\frac{(\tau_\ell - \mu x_i)^2}{2 \sigma^2}}$. We are led to
\[
I \geqslant \frac{|f_\theta(\tau_\ell)|}{2A} \frac{1}{\sqrt{e}} e^{-\frac{(\tau_\ell - \mu x_i)^2}{2 \sigma^2}} .
\]
Then
\[
n \mathcal{R}_n(\theta) \geqslant \frac{|f_\theta(\tau_\ell)|^3}{8\sqrt{2\pi e}  A \sigma} e^{-\frac{(\tau_\ell - \mu x_i)^2}{2 \sigma^2}} .
\]
We deduce, still in the case where $[-\sigma, \sigma]$ is not included in $D$, that
\begin{align*}
|f_\theta(\tau_\ell)| e^{-\frac{(\tau_\ell - \mu x_i)^2}{2 \sigma^2}} &\leqslant |f_\theta(\tau_\ell)| e^{-\frac{(\tau_\ell - \mu x_i)^2}{6 \sigma^2}} \\
&= \big( |f_\theta(\tau_\ell)|^3 e^{-\frac{(\tau_\ell - \mu x_i)^2}{2 \sigma^2}} \big)^{\frac{1}{3}}  \\
&\leqslant \big( 8\sqrt{2\pi e} A \sigma n \mathcal{R}_n(\theta)\big)^{\frac{1}{3}} \\
&= 2 \big( \sqrt{2\pi e} A \sigma n \mathcal{R}_n(\theta)\big)^{\frac{1}{3}} .
\end{align*}
Putting both cases together, we obtain
\[
|f_\theta(\tau_\ell)| e^{-\frac{(\tau_\ell - \mu x_i)^2}{2 \sigma^2}} \leqslant 2 \max \Big(\sqrt{2 n \mathcal{R}_n(\theta)}, \big( \sqrt{2\pi e} A \sigma n \mathcal{R}_n(\theta)\big)^{\frac{1}{3}} \Big) .
\]
We conclude, coming back to \eqref{eq:tech-proof-prop6-2}, that
\begin{align*}
\frac4m \textnormal{TV}_{\pi}^{(1)}&(s_\theta) \\&\leqslant \lambda_{\max}(\nabla_\theta^2\cR_n(\theta))+\frac{2}m \sqrt{\cR_n(\theta)} + \frac{2 \sqrt{2} A}{\sqrt{\pi} m \sigma} \max \Big(\sqrt{2 n \mathcal{R}_n(\theta)}, \big( \sqrt{2\pi e} A \sigma n \mathcal{R}_n(\theta)\big)^{\frac{1}{3}} \Big) .
\end{align*}
This shows the first statement of the proposition. Finally, if $\theta=\theta^\star$ is a linearly stable minimum of~$\mathcal{R}_n$, we apply \eqref{eq:bound-lambdamax-eta} to get the second inequality.

\subsection{Proof of Proposition \ref{prop:second-derivative-lower-bound}}

We start by showing by contradiction that for any $1\leqslant i\leqslant n$, there exists $a_i\in [\mu x_i-\frac{\mu\Delta}2, \mu x_i]$ and $b_i\in[\mu x_{i}, \mu x_{i}+\frac{\mu\Delta}2]$ such that $s_{\theta^\star}(a_i)>0$ and $s_{\theta^\star}(b_i)<0$. If, for all $y\in[\mu x_i-\frac{\mu\Delta}2, \mu x_i]$, one has $s_{\theta^\star}(y)\leqslant0$, then 
\begin{align*}
n\cR_n(\theta^\star)&\geqslant\E_{Y\sim\cN (\mu x_i, \sigma^2)}\big[(s_{\theta^\star}(Y)-\frac{1}{\sigma^2}(\mu x_i-Y))^2\big]\\
&\geqslant\frac{1}{\sqrt{2\pi\sigma^2}}\int_{\mu x_i-\frac{\mu\Delta}2}^{\mu x_i} \big(s_{\theta^\star}(y) - \frac{1}{\sigma^2}(\mu x_i-y)\big)^2e^{-\frac{(y-\mu x_i)^2}{2\sigma^2}}dy\\
&\geqslant \frac{1}{\sqrt{2\pi\sigma^2}}\int_{\mu x_i-\frac{\mu\Delta}2}^{\mu x_i} \frac{1}{\sigma^4}(\mu x_i-y)^2e^{-\frac{(y-\mu x_i)^2}{2\sigma^2}}dy\\
&=\frac{1}{\sigma^4\sqrt{2\pi\sigma^2}}\int_{-\frac{\mu\Delta}2}^0 y^2e^{-\frac{y^2}{2\sigma^2}}dy\\
&= \frac{1}{\sigma^2\sqrt{2\pi\sigma^2}}\Big(\frac{\mu\Delta}{2}e^{-\frac{\mu^2\Delta^2}{8\sigma^2}} + \int_{-\frac{\mu\Delta}2}^0e^{-\frac{y^2}{2\sigma^2}}dy\Big)\\
&\geqslant \frac{1}{\sigma^2\sqrt{2\pi\sigma^2}}\frac{\mu\Delta}{2}e^{-\frac{\mu^2\Delta^2}{8\sigma^2}} +\frac{1}{\sigma^2}\big(\frac12-e^{-\frac{\mu^2\Delta^2}{8\sigma^2}}\big),
\end{align*}
where we integrate by parts to derive the second last equation, and then use a tail bound of the Gaussian distribution \citep{gordon1941values}.
Since $\Delta \geqslant 8 \frac{\sigma}{\mu}$, we get
\[
n\cR_n(\theta^\star) \geqslant \frac{1}{\sigma^2}\big(\frac12-e^{-\frac{\mu^2\Delta^2}{8\sigma^2}}\big) \geqslant \frac{1}{\sigma^2}\big(\frac12-e^{-8}\big) \geqslant \frac{1}{4 \sigma^2},
\]
which is a contradiction with $\cR_n(\theta^\star)\leqslant \frac{1}{16 n\sigma^2}$. Thus, there must exist $a_i\in\big[\mu x_i-\frac{\mu\Delta}2, \mu x_i\big]$ such that $s_{\theta^\star}(a_i)\geqslant 0$. A similar argument proves the existence of $b_i$. 
Hence, for every $1\leqslant i\leqslant n-1$, there exists $c_i\in[b_i, a_{i+1}]\subset[\mu x_i, \mu x_{i+1}]$ such that $s'_{\theta^\star}(c_i)\geqslant0$. 

Assume now that for all $y\in[-\frac{\mu\Delta}{2}+\mu x_i, \frac{\mu\Delta}{2}+\mu x_i]$, we have 
\[
s_{\theta^\star}'(y)>-\frac{1}{\sigma^2}+\sqrt{\frac{2 n\cR_n(\theta^\star)}{\sigma^2}},
\]
and aim again at reaching a contradiction. By applying Lemma \ref{lem:bound-gaussian-derivative} with $f(x) = s_{\theta^\star}(x+\mu x_i)$, $\beta=\frac{1}{\sigma^2}$, $\gamma = \frac{1}{\sigma^2}-\sqrt{\frac{2 n\cR_n(\theta^\star)}{\sigma^2}}$ and $\delta=\frac{\mu\Delta}{2}$ we have
\begin{align*}
n\cR_n(\theta^\star)&>\E_{y\sim\cN(\mu x_i, \sigma^2)}[(s_{\theta^\star}(y)+\frac{1}{\sigma^2}(y - \mu x_i))^2]\\
&=\E_{z\sim\cN(0, \sigma^2)}[(f(z)+\frac{1}{\sigma^2}z)^2]\\
&\geqslant \sigma^2\frac{2 n\cR_n(\theta^\star)}{\sigma^2}\Big(1-2\Big(\frac{\mu\Delta}{2\sqrt{2\pi\sigma^2}}+1\Big)e^{-\frac{\mu^2\Delta^2}{8\sigma^2}}\Big) .%
\end{align*}
Applying Lemma \ref{lem:local-function-bound} to $x \mapsto x e^{-\frac{x^2}{2}}$ at $x = \frac{\mu \Delta}{4 \sigma} \geqslant 2$, we obtain that $\frac{\mu\Delta}{4 \sigma}e^{-\frac{\mu^2\Delta^2}{8\sigma^2}} \leqslant e^{-2}$, and thus
\[
1-2\Big(\frac{\mu\Delta}{2\sqrt{2\pi\sigma^2}}+1\Big)e^{-\frac{\mu^2\Delta^2}{8\sigma^2}} \geqslant 1 - \frac{2 \sqrt{2}}{\sqrt{\pi}} e^{-2} - 2 e^{-8} \geqslant \frac{1}{2} .
\]
We thus obtain $n\cR_n(\theta^\star) > n\cR_n(\theta^\star)$, 
which is a contradiction. So, there must exist $y_i\in[-\frac{\mu\Delta}2+\mu x_i, \frac{\mu\Delta}2+\mu x_i]$ such that
\[
s'_{\theta^\star}(y_i)\leqslant -\frac{1}{\sigma^2} + \sqrt{\frac{2 n\cR_n(\theta^\star)}{\sigma^2}} .
\]
Note that $[y_i, c_i]\cap[y_{i+2}, c_{i+2}]=\emptyset$. Let $x_{(n/4)}$ be the smallest $x_i$ such that $i>n/4$ and let $x_{(3n/4)}$ be the largest $x_i$ such that $i<3n/4$. Then, using arguments similar to those employed in the proof of Theorem \ref{thm:empirical-score-irrgularity},
\begin{align*}
\textnormal{TV}_{\pi}^{(1)}(s_{\theta^\star}) &\geqslant \sum_{\frac n4<2i<\frac{3n}{4}} \big(\min_{y\in[x_{(n/4)}, x_{(3n/4)}]} \pi(y;\mu,\sigma)\big)\Big|\int_{y_{2i}}^{c_{2i}}s''_{\theta^\star}(y)dy\Big|\\
&\geqslant \sum_{\frac n4<2i<\frac{3n}{4}} \frac{\mu n\Delta}{128}\Big(\frac12-e^{-\frac{\mu^2\Delta^2}{2\sigma^2}}\Big)\Big(\frac{1}{\sigma^2}-\sqrt{\frac{2 n\cR_n(\theta^\star)}{\sigma^2}}\Big)\\
&\geqslant\frac{\mu n^2\Delta}{1024}\Big(\frac12-e^{-\frac{\mu^2\Delta^2}{2\sigma^2}}\Big)\Big(\frac{1}{\sigma^2}-\sqrt{\frac{2 n\cR_n(\theta^\star)}{\sigma^2}}\Big),
\end{align*}
where the last inequality utilizes that there are at least $n/4$ points between $\lceil n/4 \rceil + 1$ and $\lfloor 3n/4 \rfloor - 1$ (for $n \geqslant 10$), so we sum over at least $n/8$ points given that we consider one point out of two. Then, using our assumption $\frac{\mu \Delta}{\sigma} \geqslant 8$ and $\cR_n(\theta^\star)\leqslant \frac{1}{16n\sigma^2}$, we obtain
\begin{align*}
\textnormal{TV}_{\pi}^{(1)}(s_{\theta^\star}) &\geqslant \frac{\mu n^2\Delta}{1024} (\frac{1}{2} - e^{-32}) \Big(\frac{1}{\sigma^2}-\frac{1}{2\sqrt{2}\sigma^2}\Big)  \geqslant \frac{\mu n^2\Delta}{2^{11} \sigma^2} .
\end{align*}

\subsection{Proof of Theorem \ref{thm:main-without-d2}}

We reason by contraposition, that is, we assume that
\[
\cR_n(\theta^\star)  - \cR_n(s^\star) \leqslant \frac{\pi n^5 \mu^{3} \Delta^{3}}{2^{36} e^{1/2} A^4 \sigma^4} ,
\]
and show that it implies that $\eta \leqslant \frac{2^{12} \sigma^2}{\mu n^2 \Delta}$.
For $\sigma \leqslant \sigma_1 := \frac{\mu \Delta}{8}$, we can apply Corollary \ref{cor:regularity-empirical-optimal-score}, which gives
\[
\mathcal{R}_n(s^\star) \leqslant \frac{4 \mu^2 (x_n-x_1)^2}{\sigma^4} e^{-\frac{\mu^2 \Delta^2}{32 \sigma^2}} . %
\]
Thus
\[
\cR_n(\theta^\star) \leqslant \frac{\pi n^5 \mu^{3} \Delta^{3}}{2^{36} e^{1/2} A^4 \sigma^4}  + \frac{4 \mu^2 (x_n-x_1)^2}{\sigma^4} e^{-\frac{\mu^2 \Delta^2}{32 \sigma^2}} . 
\]
Let us show that this implies
\begin{align}   \label{eq:tech-proof-thm}
\cR_n(\theta^\star) \leqslant \frac{\pi n^5 \mu^{3} \Delta^{3}}{2^{35} e^{1/2} A^4 \sigma^4} .
\end{align}
By rearranging terms, one can see that this holds as soon as
\[
e^{-\frac{\mu^2 \Delta^2}{32 \sigma^2}} \leqslant \frac{\pi n^5 \mu}{2^{38} e^{1/2} (x_n-x_1)^2 A^4} .
\]
Recalling that $A$ grows polynomially fast with $1/\sigma$, we observe that the left-hand side of the previous inequality decays exponentially fast when $\sigma \to 0$, while the right-hand side decays polynomially fast. This implies the existence of some $\sigma_2$ depending on the training data and on $\mu$ such that this inequality holds true for $\sigma \leqslant \sigma_2$. 

Next, observe that \eqref{eq:tech-proof-thm} implies in particular that, for $\sigma \leqslant \sigma_3 := \frac{1}{n}$,
\begin{align}   \label{eq:tech-proof-thm-2}
\cR_n(\theta^\star) \leqslant \frac{1}{16 n \sigma^2}.
\end{align}
This enables us to apply Proposition \ref{prop:second-derivative-lower-bound} to $\theta = \theta^\star$, which entails that, for $\sigma \leqslant \sigma_1$,
\[
\textnormal{TV}_{\pi}^{(1)}(s_{\theta^\star}) \geqslant \frac{\mu n^2\Delta}{2^{11} \sigma^2} 
\]
Combining this lower bound with the upper bound of Proposition \ref{prop:second-derivative-upper-bound}, we obtain that
\begin{align*}
\frac{1}{2\eta}+\frac{\sqrt{\cR_n(\theta^\star)}}{2}+\frac{A}{ \sqrt{2 \pi} \sigma} \max \Big(\sqrt{2 n \mathcal{R}_n(\theta^\star)}, \big( \sqrt{2\pi e} A \sigma n \mathcal{R}_n(\theta^\star)\big)^{\frac{1}{3}} \Big) \geqslant \frac{\mu n^2\Delta}{2^{11} \sigma^2} .
\end{align*}
Note that
\begin{align*}
\sqrt{2 n \mathcal{R}_n(\theta^\star)} \leqslant \big( \sqrt{2\pi e} A \sigma n \mathcal{R}_n(\theta^\star)\big)^{\frac{1}{3}} &\Leftrightarrow 8 n^3 \mathcal{R}_n(\theta^\star)^3 \leqslant 2 \pi e A^2 \sigma^2 n^2 \mathcal{R}_n(\theta^\star)^2 \\
&\Leftrightarrow \mathcal{R}_n(\theta^\star) \leqslant \frac{\pi e}{4} \frac{A^2 \sigma^2}{n} ,
\end{align*}
which holds true by \eqref{eq:tech-proof-thm-2}.
Hence, we obtain
\[
\frac{1}{2\eta}+\frac{\sqrt{\cR_n(\theta^\star)}}{2}+\frac{A}{ \sqrt{2 \pi} \sigma} \big( \sqrt{2\pi e} A \sigma n \mathcal{R}_n(\theta^\star)\big)^{\frac{1}{3}}  \geqslant \frac{\mu n^2\Delta}{2^{11} \sigma^2} .
\]
Rewriting the third term, we have
\[
\frac{1}{2\eta}+\frac{\sqrt{\cR_n(\theta^\star)}}{2}+\frac{e^{1/6} A^{4/3} n^{1/3} \mathcal{R}_n(\theta^\star)^{1/3}}{ (2 \pi)^{1/3} \sigma^{2/3}}  \geqslant \frac{\mu n^2\Delta}{2^{11} \sigma^2} .
\]
By \eqref{eq:tech-proof-thm}, recalling again that $A \geqslant \frac{C_n}{\sigma^6}$, observe that there exists $\sigma_4$ such that, for $\sigma \leqslant \sigma_4$,
\[
\frac{\sqrt{\cR_n(\theta^\star)}}{2} \leqslant \frac{\mu n^2\Delta}{2^{13} \sigma^2} .
\]
Thus, using again \eqref{eq:tech-proof-thm}, we get
\[
\frac{1}{2 \eta} + \frac{\mu n^2\Delta}{2^{13} \sigma^2} + \frac{\mu n^2\Delta}{2^{12} \sigma^2} \geqslant \frac{\mu n^2\Delta}{2^{11} \sigma^2} ,
\]
which implies that $\eta \leqslant \frac{2^{12} \sigma^2}{\mu n^2 \Delta}$, thereby concluding the proof with $\sigma_0 = \min(\sigma_1, \sigma_2, \sigma_3, \sigma_4)$.

\section{Technical lemmas}
\label{sec:technical-lemma}

The first lemma relates the derivatives of $s^\star$ with the moments of $W$.
\begin{lemma} \label{lem:derivative-and-gradient}
Let $s^\star$ and $W$ be defined as in Sections \ref{sec:denoising-glance} and \ref{sec:regularity-empirical-score}. Then we have
\[
s^\star{}'(y;\mu,\sigma)=\frac{1}{\sigma^2 }\Big(-1+\frac{ 
\mu^2}{\sigma^2 }\V[W(y;\mu,\sigma)]\Big),
\]
and
\begin{equation*}
s^\star{}''(y;\mu,\sigma) =\frac{\mu^3}{\sigma^6}\E[(W(y;\mu,\sigma) - \E[W(y;\mu,\sigma)])^3].
\end{equation*}
\end{lemma}
\begin{proof}
We start by calculating the derivatives of $\alpha_j$ with respect to $y$. Observe that 
\begin{align*}
&\alpha_j'(y;\mu,\sigma) = -\frac{e^{-\frac{(y-\mu  x_j)^2}{2\sigma^2}}\sum_{i=1}^n (-y+ \mu x_i )e^{-\frac{(y-\mu  x_i)^2}{2\sigma^2}}}{\sigma^2 \big(\sum_{i=1}^n e^{-\frac{(y-\mu  x_i)^2}{2\sigma^2}}\big)^2} +\frac{(-y+ \mu x_j)e^{-\frac{(y-\mu  x_j)^2}{2\sigma^2}}}{\sigma^2 \sum_{i=1}^n e^{-\frac{(y-\mu  x_i)^2}{2\sigma^2}}},\\
&=\frac1{\sigma^2 }\Big(y\alpha_j(y;\mu,\sigma)- \mu \alpha_j(y;\mu,\sigma)\sum_{i=1}^n x_i \alpha_i(y;\mu,\sigma)-y\alpha_j(y;\mu,\sigma)+  \mu x_j\alpha_j(y;\mu,\sigma)\Big),\\
&=\frac{  \alpha_j(y;\mu,\sigma)}{\sigma^2 }\Big(\mu x_j-\mu\sum_{i=1}^n x_i \alpha_i(y;\mu,\sigma)\Big),\\
&=\frac{\mu  \alpha_j(y;\mu,\sigma)}{\sigma^2 }\big(x_j-\mathbb{E}[W(y;\mu,\sigma)]\big).
\end{align*}
In addition, we may also compute $\alpha_j''(y;\mu,\sigma)$ as follows
\begin{align*}
\alpha_j''(y;\mu,\sigma) &= \frac{d}{dy} \Big(\frac{  \mu \alpha_j(y;\mu,\sigma)}{\sigma^2 }\big(x_j-\sum_{i=1}^n x_i \alpha_i(y;\mu,\sigma)\big)\Big),\\
&=\frac{ \mu \alpha_j'(y;\mu,\sigma)}{\sigma^2 }\Big(x_j-\sum_{i=1}^n x_i \alpha_i(y;\mu,\sigma)\Big) - \frac{  \mu\alpha_j(y;\mu,\sigma)}{\sigma^2 }\sum_{i=1}^n x_i \alpha_i'(y;\mu,\sigma),\\
&= \frac{ \mu^2 \alpha_j(y;\mu,\sigma)}{\sigma^4 }\Big(x_j - \sum_{i=1}^n x_i \alpha_i(y;\mu,\sigma)\Big)^2\\
&\qquad-\frac{\mu^2\alpha_j(y;\mu,\sigma)}{\sigma^2 }\sum_{i=1}^n x_i\frac{  \alpha_i(y;\mu,\sigma)}{\sigma^2 }\Big(x_i - \sum_{i'=1}^n x_{i'} \alpha_{i'}(y;\mu,\sigma)\Big),\\
&=\frac{\mu^2\alpha_j(y;\mu,\sigma)}{\sigma^2 }\Big(x_j-\E[W(y;\mu,\sigma)]\Big)^2\\
&\qquad-\frac{\mu^2\alpha_j(y;\mu,\sigma)}{\sigma^4 }\Big(\E[W^2(y;\mu,\sigma)]-\E^2[W(y;\mu,\sigma)]\Big),\\
&=\frac{\mu^2\alpha_j(y;\mu,\sigma)}{\sigma^4 }\big(x_j^2-2x_j\E[W(y;\mu,\sigma)]-\E[W^2(y;\mu, \sigma)]+2\E^2[W(y;\mu, \sigma)]\big)
\end{align*}
Consequently,
\begin{align*}
s^\star{}'(y;\mu,\sigma)&=\frac{1}{\sigma^2 }\Big(-1+\mu\sum_{i=1}^n  x_i \alpha_i'(y; \mu, \sigma)\Big)\nonumber\\
&=\frac{1}{\sigma^2 }\Big(-1+\frac{\mu^2}{\sigma^2} \Big( \sum_{i=1}^n \alpha_i(y;\mu,\sigma) x_i^2 - \mathbb{E}[W(y;\mu,\sigma)] \sum_{i=1}^n \alpha_i(y;\mu,\sigma) x_i \Big) \Big)\nonumber\\
&=\frac{1}{\sigma^2 }\Big(-1+\frac{\mu^2}{\sigma^2 }\big(\mathbb{E}[W^2(y;\mu,\sigma)]-\mathbb{E}^2[W(y;\mu,\sigma)]\big)\Big)\nonumber\\
&=\frac{1}{\sigma^2 }\Big(-1+\frac{ 
\mu^2}{\sigma^2 }\V[W(y;\mu,\sigma)]\Big),
\end{align*}
and
\begin{align*}
s^\star{}''&(y;\mu,\sigma) = \frac{\mu}{\sigma^2}\sum_{i=1}^nx_i\alpha_i''(y;\mu,\sigma),\\
&=\frac{\mu}{\sigma^2}\sum_{i=1}^n\frac{\mu^2x_i\alpha_i(y;\mu,\sigma)}{\sigma^4 }\big(x_i^2-2x_i\E[W(y;\mu,\sigma)]-\E[W^2(y;\mu,\sigma)]+2\E^2[W(y;\mu, \sigma)]\big),\\
&=\frac{\mu^3}{\sigma^6}(\E[W^3(y;\mu,\sigma)]-2\E[W^2(y;\mu,\sigma)]\E[W(y;\mu,\sigma)])\\
&\qquad+\frac{\mu^3}{\sigma^6}(-\E[W^2(y;\mu,\sigma)]\E[W(y;\mu,\sigma)]+2\E^3[W(y;\mu,\sigma)]),\\
&=\frac{\mu^3}{\sigma^6}(\E[W^3(y;\mu,\sigma)]-3\E[W^2(y;\mu,\sigma)]\E[W(y;\mu,\sigma)]+2\E^3[W(y;\mu,\sigma)]),\\
&=\frac{\mu^3}{\sigma^6}\E[(W(y;\mu,\sigma)-\E[W(y;\mu,\sigma)])^3].
\end{align*}
This concludes the proof.
\end{proof}

The next lemma bounds the total variation of the derivative of $s^\star$.
\begin{lemma}
\label{lem:asymptotic-alpha}
Let $(x_n)_+=\max(0, x_n)$ and $(x_1)_-=\min(0, x_1)$. Then
\[\int_\R|s^\star{}''(y;\mu,\sigma)|dy\leqslant\frac{4\mu^2(x_n-x_1)^3}{\sigma^6}\Big(\mu^2((x_n)_+-(x_1)_-)+\frac{2(n-1)\sigma^2}{\Delta}\Big).
\]
\end{lemma}
\begin{proof}
We start by proving  that, for  $1\leqslant i\leqslant n-1$ and $y\geqslant 2\mu (x_n)_+$, one has
\[
\alpha_i(y;\mu,\sigma)\leqslant e^{-\frac{y\mu\Delta}{2\sigma^2}} \quad \textnormal{and} \quad |s^\star{}''(y;\mu, \sigma)| \leqslant \frac{2\mu^3(n-1)(x_n-x_1)^3}{\sigma^6}e^{-\frac{y\mu\Delta}{2\sigma^2}}.
\]
Observe that
\begin{align*}
\alpha_i(y;\mu,\sigma) = \frac{e^{-\frac{(y-\mu x_{i})^2}{2\sigma^2}}}{\sum_{i'=1}^n e^{-\frac{(y-\mu x_{i'})^2}{2\sigma^2}}}\leqslant\frac{e^{-\frac{(y-\mu x_{i})^2}{2\sigma^2}}}{e^{-\frac{(y-\mu x_n)^2}{2\sigma^2}}}=e^{\frac{-2y\mu(x_n-x_i)+\mu^2(x_n^2-x_i^2)}{2\sigma^2}}.
\end{align*}
Since $y\geqslant 2\mu (x_n)_+\geqslant2\mu x_n$ implies $y\geqslant \mu(x_n+x_i)$, noticing that $x_n-x_i\geqslant\Delta$, we have
\[
-2y\mu(x_n-x_i)+\mu^2(x_n^2-x_i^2) = \mu(x_n-x_i)(-2y+\mu(x_n+x_i))\leqslant-y\mu(x_n-x_i)\leqslant-y\mu\Delta,
\]
where we also used the fact that $y\geqslant0$ in the last inequality. It follows that 
\begin{equation}    \label{eq:tech-lemma-12-1}
\alpha_i(y;\mu,\sigma)\leqslant e^{\frac{-2y\mu(x_n-x_i)+\mu^2(x_n^2-x_i^2)}{2\sigma^2}}\leqslant e^{-\frac{y\mu\Delta}{2\sigma^2}}.    
\end{equation}
To upper bound $|s^\star{}''(y;\mu,\sigma)|$, we first remark that $W(y;\mu,\sigma)$ takes value in $\{x_1,\dots, x_n\}$. Hence, for all $1 \leqslant i \leqslant n$, $|x_i-\E[W(y;\mu,\sigma)]|\leqslant x_n-x_1$. Applying Lemma \ref{lem:derivative-and-gradient}, we are led to
\begin{align}
|s^\star{}''(y;\mu,\sigma)|&\leqslant \frac{\mu^3}{\sigma^6}\E\big[|W(y;\mu,\sigma)-\E[W(y;\mu,\sigma)]|^3\big]~\label{eq:sstar-bound-second-derivative}\\
&= \frac{\mu^3}{\sigma^6}\sum_{i=1}^{n-1} \big|x_i - \E[W(y;\mu,\sigma)]\big|^3\alpha_i(y;\mu,\sigma)+\frac{\mu^3}{\sigma^6}\big|x_n - \E[W(y;\mu,\sigma)]\big|^3\alpha_n(y;\mu,\sigma)\nonumber\\
&\leqslant \frac{\mu^3}{\sigma^6} \sum_{i=1}^{n-1} (x_n-x_1)^3e^{-\frac{y\mu\Delta}{2\sigma^2}} + \frac{\mu^3}{\sigma^6} \Big|x_n-\sum_{i=1}^nx_i\alpha_i(y;\mu,\sigma)\Big|^3\alpha_n(y;\mu,\sigma)\nonumber\\
&\leqslant \frac{\mu^3}{\sigma^6} (n-1) (x_n-x_1)^3e^{-\frac{y\mu\Delta}{2\sigma^2}} +  \frac{\mu^3}{\sigma^6} \Big|\sum_{i=1}^{n-1} (x_n-x_i)\alpha_i(y;\mu,\sigma)\Big|^3\nonumber\\
&\leqslant  \frac{\mu^3}{\sigma^6} (n-1)(x_n-x_1)^3e^{-\frac{y\mu\Delta}{2\sigma^2}} +  \frac{\mu^3}{\sigma^6} \sum_{i=1}^{n-1} (x_n-x_i)^3\alpha_i(y;\mu,\sigma)\nonumber\\
&\leqslant \frac{2\mu^3(n-1)(x_n-x_1)^3}{\sigma^6}e^{-\frac{y\mu\Delta}{2\sigma^2}},\nonumber
\end{align}
where, in the second to last line we use the fact that $x\mapsto x^3$ is convex on $\R_+$ and $x_n-x_i>0$ for all $1\leqslant i\leqslant n-1$, and the last inequality follows from \eqref{eq:tech-lemma-12-1}.

A similar argument apply for $y<2\mu (x_1)_-$ and $2\leqslant i\leqslant n$. In this case,
\[
\alpha_i(y;\mu,\sigma)\leqslant e^{\frac{y\mu\Delta}{2\sigma^2}}\quad\text{and}\quad
|s^\star{}''(y;\mu, \sigma)| \leqslant \frac{2\mu^3(n-1)(x_n-x_1)^3}{\sigma^6}e^{\frac{y\mu\Delta}{2\sigma^2}}.
\]

We can now proceed to bounding $\int_\R|s^\star{}''(y;\mu,\sigma)|dy$. To this aim, we first split the integral as follows
\begin{align}
\begin{split}   \label{eq:tech-lemma-12-2}
 &\int_\R |s^\star{}''(y;\mu,\sigma)|dy\\
&=\int_{2\mu(x_1)_-}^{2\mu (x_n)_+} |s^\star{}''(y;\mu,\sigma)|dy+\int_{-\infty}^{2\mu(x_1)_-}|s^\star{}''(y;\mu,\sigma)|dy+\int_{2\mu (x_n)_+}^\infty|s^\star{}''(y;\mu,\sigma)|dy.   
\end{split}
\end{align}
Similar to the argument in \eqref{eq:sstar-bound-second-derivative}, we have $|s^\star{}''(y;\mu,\sigma)| \leqslant  \frac{\mu^3}{\sigma^6}\E[|W(y;\mu,\sigma)-\E[W(y;\mu,\sigma)]|^3]\leqslant\frac{\mu^3}{\sigma^6}(x_n-x_1)^3$. Therefore, 
\[
\int_{2\mu(x_1)_-}^{2\mu (x_n)_+}  |s^\star{}''(y;\mu,\sigma)|dy\leqslant\frac{4\mu^4(x_n-x_1)^3((x_{n})_+-(x_{1})_-)}{\sigma^6}.
\]
To bound the last two terms on the right-hand side of \eqref{eq:tech-lemma-12-2}, we use the previous derivations, and see that
\begin{align*}
\int_{-\infty}^{2\mu (x_1)_-}|s^\star{}''(y;\mu,\sigma)|dy&\leqslant\frac{2\mu^3(n-1)(x_n-x_1)^3}{\sigma^6}\int_{-\infty}^{2\mu (x_1)_-}e^{\frac{y\mu\Delta}{2\sigma^2}}dy\\
&=\frac{2\mu^3(n-1)(x_n-x_1)^3}{\sigma^6} \frac{2\sigma^2}{\mu\Delta}e^{\frac{\mu^2\Delta(x_1)_-}{\sigma^2}}\\
&\leqslant\frac{4\mu^2(n-1)(x_n-x_1)^3}{\sigma^4\Delta},
\end{align*}
since $(x_1)_-\leqslant0$ implies $e^{\frac{\mu^2\Delta(x_1)_-}{\sigma^2}}\leqslant 1$. Similarly,
\[
\int_{2\mu (x_n)_+}^\infty|s^\star{}''(y;\mu,\sigma)|dy\leqslant\frac{4\mu^2(n-1)(x_n-x_1)^3}{\sigma^4\Delta}.
\]
Putting everything together, we have 
\begin{align*}
\int_\R |s^\star{}''(y;\mu,\sigma)|dy&\leqslant \frac{4\mu^4(x_n-x_1)^3((x_{n})_+-(x_{1})_-)}{\sigma^6}+\frac{8\mu^2(n-1)(x_n-x_1)^3}{\sigma^4\Delta}\\
&\leqslant \frac{4\mu^2(x_n-x_1)^3}{\sigma^6}\Big(\mu^2((x_n)_+-(x_1)_-)+\frac{2(n-1)\sigma^2}{\Delta}\Big),
\end{align*}
which is the desired result.
\end{proof}

The final two lemmas are technical properties of moments of Gaussian distributions.
\begin{lemma}
\label{lem:local-function-bound}
Let $k$ be a positive integer. Then the function $f_k:\R\to\R$ defined by $f_k(x) = x^ke^{-\frac{Ax^2}2}$ is strictly increasing on $(0, \sqrt{\frac{k}{A}})$ and strictly decreasing on $(\sqrt{\frac{k}{A}}, \infty)$.
\end{lemma}
\begin{proof}
Consider the derivative of $f_k$, which is given by
\begin{align*}
f_k'(x) = kx^{k-1}e^{-\frac{Ax^2}2} - Ax^{k+1}e^{-\frac{Ax^2}{2}} =(k - Ax^2)x^{k-1} e^{-\frac{Ax^2}{2}}.
\end{align*}
Clearly, $f_k'<0$ on $(\sqrt{\frac{k}{A}}, \infty)$ and $f_k'>0$ on $(0, \sqrt{\frac{k}{A}})$.
\end{proof}
\begin{lemma}
\label{lem:bound-gaussian-derivative}
Let $\beta, \gamma, \delta$ be positive numbers such that $\beta>\gamma$. Let $f:\R\to\R$ be a differentiable function satisfying $f'(x)\geqslant-\gamma$ for $x\in[-\delta, \delta]$. Then 
\begin{align*}
\E_{x\sim\mathcal{N}(0, \sigma^2)}\big[(f(x)+\beta x)^2\big]&\geqslant \frac{(\beta-\gamma)^2}{\sqrt{2\pi\sigma^2}}\int_{-\delta}^\delta x^2e^{-\frac{x^2}{2\sigma^2}}dx\\
&\geqslant \sigma^2(\beta-\gamma)^2\Big(1-2\Big(\frac{\delta}{\sqrt{2\pi\sigma^2}}+1\Big)e^{-\frac{\delta^2}{2\sigma^2}}\Big).
\end{align*} 
\end{lemma}
\begin{proof}
Notice that 
\begin{align*}
\E_{Z\sim\cN(0, \sigma^2)}\big[(f(x)+\beta x)^2\big] & \geqslant \frac{1}{\sqrt{2\pi\sigma^2}}\int_{[-\delta, \delta]} (f(y)+\beta y)^2e^{-\frac{y^2}{2\sigma^2}}dy\\
&\geqslant  \frac{1}{\sqrt{2\pi\sigma^2}}\int_{[-\delta, \delta]} (-\gamma y+f(0)+\beta y)^2e^{-\frac{y^2}{2\sigma^2}}dy.
\end{align*}
The last term reads
\begin{align*}
&\frac{1}{\sqrt{2\pi\sigma^2}}\int_{-\delta}^\delta \big((\beta-\gamma)^2y^2+2(\beta-\gamma)f(0)y+f^2(0)\big)e^{-\frac{y^2}{2\sigma^2}}dy\\
&\quad =\frac{1}{\sqrt{2\pi\sigma^2}}\int_{-\delta}^\delta \big((\beta-\gamma)^2y^2+f^2(0)\big)e^{-\frac{y^2}{2\sigma^2}}dy\\
&\quad \geqslant \frac{1}{\sqrt{2\pi\sigma^2}}\int_{-\delta}^\delta(\beta-\gamma)^2y^2e^{-\frac{y^2}{2\sigma^2}}dy.\\
&\quad=\frac{-\sigma^2}{\sqrt{2\pi\sigma^2}}\int_{-\delta}^\delta(\beta-\gamma)^2y\frac{de^{-\frac{y^2}{2\sigma^2}}}{dy}dy,\\
&\quad=\Big[\frac{-\sigma^2}{\sqrt{2\pi\sigma^2}}(\beta-\gamma)^2ye^{-\frac{y^2}{2\sigma^2}}\Big]_{-\delta}^\delta+\frac{\sigma^2}{\sqrt{2\pi\sigma^2}}\int_{-\delta}^\delta(\beta-\gamma)^2e^{-\frac{y^2}{2\sigma^2}}dy\\
&\quad=\sigma^2(\beta-\gamma)^2\Big(\frac{-2\delta}{\sqrt{2\pi\sigma^2}}e^{-\frac{\delta^2}{2\sigma^2}} + \PP_{\xi\in\cN(0, 1)}(|\sigma \xi|\leqslant \delta)\Big)\\
&\quad\geqslant \sigma^2(\beta-\gamma)^2\Big(\frac{-2\delta}{\sqrt{2\pi\sigma^2}}e^{-\frac{\delta^2}{2\sigma^2}} + 1-2e^{-\frac{-\delta^2}{2\sigma^2}}\Big)
\end{align*}
where in the fifth line we use integration by part and in the last line we use a Gaussian tail bound \citep{gordon1941values}.
\end{proof}
\section{Comments on gradient descent}
\label{app:comments-gradient-descent}
The choice of SGD is motivated by the fact that the loss is defined as an expectation \eqref{eq:score-matching-objective}, thus requiring a stochastic approximation. However our analysis also applies to an idealized scenario: GD performed directly on the population risk $\cR_n$ \eqref{eq:score-matching-objective}. 

For $j\in\mathbb N$, the standard GD update is $
\theta_{j+1} = \theta_j - m\eta\nabla \cR_n(\theta_j)
$.
Similar to the argument in Section \ref{sec:problem-setups}, our analysis primarily relies on the linearized dynamics around the linearly stable global minimizer $\theta^\star$, where $\nabla \cR_n(\theta_j)\approx\nabla\cR_n(\theta^\star) + \nabla^2\cR_n(\theta^\star)(\theta_j-\theta^\star)$. Since $\nabla\cR_n(\theta^\star) = 0$, the linearized GD updates may be written as follows: 
\[
\theta_{j+1} = \theta_j - m\eta\nabla^2\cR_n(\theta^\star)(\theta_j-\theta^\star).
\]
Therefore, the central property \eqref{eq:bound-lambdamax-eta}, which relates the learning rate to the Hessian at the optimum $\nabla^2\cR_n(\theta^\star)$, still holds in this deterministic setting. Consequently, our proof carries over and shows a lower bound on the learning rate above which GD cannot converge to the global minimizer. This demonstrates that our proof of the main result does not fundamentally depend on the stochasticity of the optimization algorithm. 
The idealized scenario of GD on the population risk effectively represents the case where the variance of the gradient estimates is zero. In other words, the non-convergence of SGD towards the global minimizer is not due to the lack of handling the variance of the gradient estimates, but rather to an (implicit) bias due to the large learning rate.
 
Furthermore, our proof does not necessarily require constant learning rate. If the learning rate varies per iteration as $\eta_j$, as long as there exists $\eta_{\min} > 0$ such that $\eta_j\geq\eta_{\min}$ for every $j\in\mathbb N$, it suffices to replace the $\eta$ in \eqref{eq:bound-lambdamax-eta} with $\eta_{\min}$. Specifically, the condition becomes:
\[
\lambda_{\max}(\nabla^2\cR_n(\theta^\star))\leq \frac{2}{m\eta_{\min}}.
\]
By reaplacing the $\eta$ with $\eta_{\min}$ in subsequent propositions and bounds, the proof remains valid.

\section{Experimental details and additional results}  \label{apx:experiments}
Our code is available at \\
\url{https://github.com/pojoowu/Prevent-Memorization-via-implicit-regularization}.
\paragraph{Model.} For all the experiments, we fix the model to be a 2-layer ReLU network with a hidden width $m=1000$. We initialize outer weights as standard Gaussian random variables, inner weights as Gaussian random variables of variance $1/d$, and the inner bias to be 0. Note that this is the standard initialization scheme of $2$-layer networks (in the feature learning regime), and differs from our theoretical setup from Section \ref{sec:problem-setups} where the inner weights are set to $\pm 1$.
\paragraph{Figure \ref{fig:exp-plot-function}.} We use a set of learning rates in $\{.5, .1, .05\}$ and with number of epochs 
$$\{5,000, 25,000, 50,000\}$$
respectively. The batch size is set to $50$. For each pair of $(\mu, \sigma)$ we generate $20$ training data by sampling the standard Gaussian distribution, and keep the training data to be the same for every learning rate. 

\paragraph{Figure \ref{fig:excess-loss-1d}.} The model is trained with different learning rates with $30$ simulations. The set of learning rates is the same as previously, with additionally the learning rate $.01$ (and  $200,000$ epochs). In each simulation, we generate $20$ training data with the standard Gaussian law and use a batch size of $50$. In addition, to estimate the excess risk 
\[
\cR_n(\theta^\star) - \cR_n(s^\star) = \frac1n\sum_{i=1}^n\E_{Y\sim \mathcal{N}(\mu x_i , \sigma^2)}\big[(s_{\theta^\star}(Y) - s^\star(Y;\mu,\sigma))^2\big],
\]
we generate $5000$ Gaussian noises for each training data to simulate the expectation.
\paragraph{Figure \ref{fig:exp-2d-memorization}.} The training data is sampled from the isotropic Gaussian distribution of standard deviation $2$, and we keep the dataset the same for training with each learning rate. We minimize by SGD over $r_\theta: \R^{d+1} \to \R^d$ the risk
\[
\int_\delta^T \frac1n\sum_{i=1}^n \E_{Z\sim\cN(0, I)}\big[(r_\theta(t, \mu(t)Y+\sigma(t) Z)-Z)^2\big],
\]
where $\mu(t) = e^{-t}$ and $\sigma(t)=\sqrt{1-e^{-2t}}$. The integral over $T$ is discretized over $100$ equally spaces times.
The batch size is set to $5,000$. For each batch element, the time $t$ is sampled among the $100$ discretization points, with a probability proportional to $\sigma(t)$. We use a set of learning rate in $\{2., 0.05\}$ and number of epochs to be $\{2\times10^4, 10^6\}$. Note that the risk differs from the one we analyze (c.f.~\eqref{eq:score-matching-objective}) by an affine transform. This is standard in practice for numerical stability reasons. Accordingly, we take the score to be $s_\theta(t,x) = -\frac{1}{\sigma(t)} r_\theta(t,x)$.
We then generate new samples using the backward Ornstein-Uhlenbeck process starting from $T=1$ and ending at $\delta =.01$. The MMD distance is calculated with the Gaussian kernel with bandwidth equal to $1$. We also checked that other metrics give qualitatively similar conclusions (MMD with other bandwidth, Wasserstein distance).
\paragraph{Figure \ref{fig:exp-dimension}.} We fix the learning rate to be $.5$ and the number of epochs to be $50,000$. We use a set of dimensions $\{2, 5, 10, 50, 80, 100, 200, 400, 1000\}$. For each dimension, we train the model with $3$ simulations and in each simulation generate the training data by sampling the isotropic Gaussian of standard deviation $2$. The score matching objective and generation procedure are the same as previously. We generate $5$ sets of data with each simulations, and the MMD is computed with the Gaussian kernel with bandwidth set to $1$. 
\begin{figure}
    \centering
    \hfill
    \includegraphics[width=0.32\linewidth]{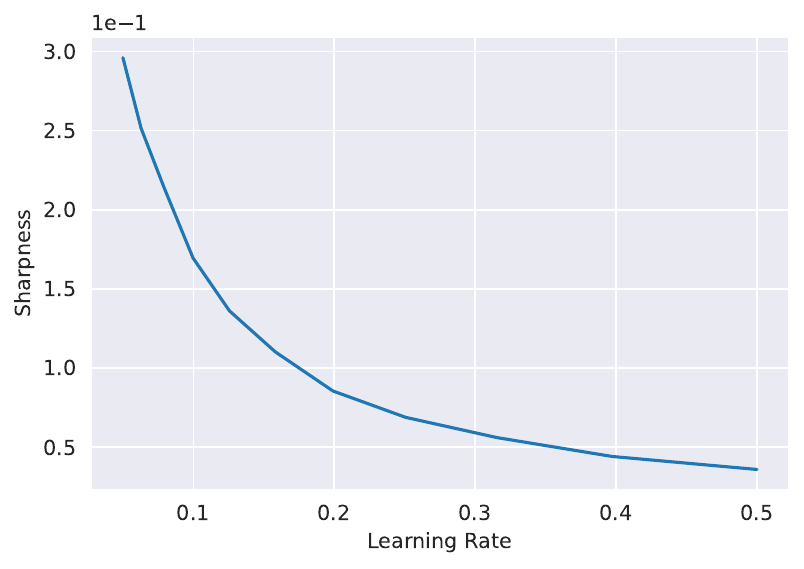}
    \hfill
    \includegraphics[width=0.32\linewidth]{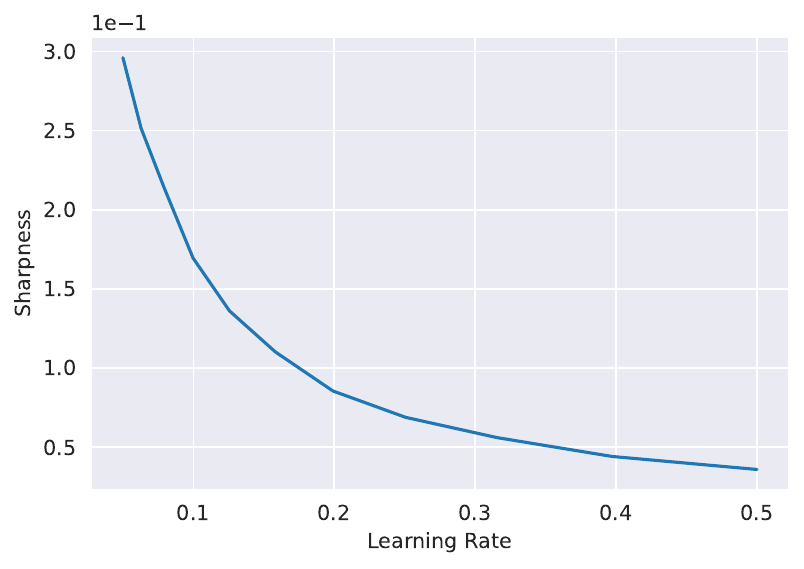}
    \includegraphics[width=0.32\linewidth]{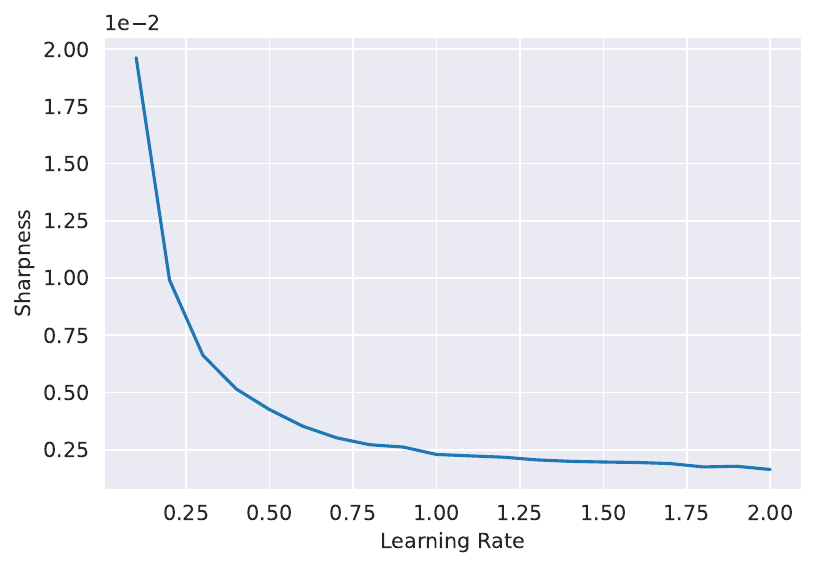}
    \caption{Largest eigenvalue of the loss Hessian (or sharpness) at the end of training, as a function of the learning rate. (left) and (middle) for the experiment of Figures \ref{fig:exp-plot-function} and \ref{fig:excess-loss-1d}, for $d=1$ and $d=10$ respectively, with $(\mu, \sigma) = (0.81, 0.57)$. (right) for the experiment of Figure \ref{fig:exp-2d-memorization}.}
    \label{fig:sharpness}
\end{figure}
\end{document}